\documentclass[12pt]{article}
\usepackage{amsmath}
\usepackage{amsfonts}
\usepackage{mathtools}
\usepackage{bigints}
\usepackage{color}
\usepackage{tabu}
\usepackage[final]{graphicx}
\usepackage{subcaption}\usepackage[letterpaper, portrait, margin=1.1in]{geometry}
\usepackage{setspace}
\usepackage{authblk}
\usepackage{float}
\restylefloat{table}
\usepackage[round]{natbib}
\usepackage{enumitem}

\def\Var{{\rm Var}}

\def\I{{\mathbf{I}}}

\def\one{{\mathbf{1}}}

\def\RR{{\mathbf{ R}}}

\newcommand{\ru}{\mbox{Uniform}}

\newcommand{\beq}{\begin{equation}}
\newcommand{\eeq}{\end{equation}}
\newcommand{\beqn}{\begin{eqnarray}}
\newcommand{\eeqn}{\end{eqnarray}}

\usepackage{amsthm}
\theoremstyle{plain}
\newtheorem{theorem}{Theorem}

\usepackage{amsfonts}
\newcommand{\overbar}[1]{\mkern 1.5mu\overline{\mkern-1.5mu#1\mkern-1.5mu}\mkern 1.5mu}

\begin{document}
\title{Likelihood Inflating Sampling Algorithm \\}
\author[1]{Reihaneh Entezari\thanks{entezari@utstat.utoronto.ca}}
\author[1]{Radu V. Craiu\thanks{craiu@ustat.toronto.edu ~~  web:~http://www.utstat.toronto.edu/craiu/}}
\author[1]{Jeffrey S. Rosenthal\thanks{jeff@math.toronto.edu ~~  web:~http://probability.ca/jeff/}}
\affil[1]{Department of Statistical Sciences, University of Toronto}
%\date{Submitted April 2016, Revised February 2017}
\maketitle

\begin{abstract}
Markov Chain Monte Carlo (MCMC) sampling from a posterior distribution  corresponding to a massive data set can be computationally prohibitive  since producing one sample requires a number of operations that is linear in the data size. In this paper, we introduce a new communication-free parallel method, the {\it Likelihood Inflating Sampling Algorithm (LISA)}, that significantly reduces computational costs by randomly splitting the dataset into smaller subsets and running MCMC methods {\em independently} in parallel on each subset using different processors. Each processor will be used to run an MCMC chain that samples  sub-posterior distributions which are defined using an ``inflated"  likelihood function.  We develop a strategy for combining the  draws from different sub-posteriors   to  study the full posterior of the Bayesian Additive Regression Trees (BART) model. The performance of the method is tested using  simulated data and a large socio-economic study. 
%The method we propose shows significant efficiency gains over the existing Consensus Monte Carlo of \cite{cons}.
\end{abstract}

\section{Introduction}
Markov Chain Monte Carlo (MCMC) methods are essential  for
sampling highly complex distributions. They are of paramount
importance in Bayesian inference as posterior distributions are
generally difficult to characterize analytically \citep[e.g.,][]{MCMChandbook, radu_rosenthal}. When the posterior distribution is based on a massive sample of size $N$, posterior sampling can be computationally prohibitive since for some widely-used samplers at least $O(N)$ operations are needed to draw one MCMC sample. Additional issues include memory and storage bottlenecks where datasets are too large to be stored on one computer. 
  
A common solution relies on  parallelizing the computation task, i.e. dividing the load among a number of parallel  {\it workers}, where a worker can be a processing unit, a computer, etc. Given the abundant availability of  processing units, such strategies can be extremely efficient as long as there is no need for frequent communication between workers. Some have discussed parallel MCMC methods \citep{wilkinson, rosenthal, laskey} such that each worker runs on the full dataset.  However, these methods do not resolve memory overload, and also face difficulties in assessing the number of burn-in iterations for  each processor. 

A truly parallel approach  is to divide the dataset into smaller groups and run parallel MCMC methods on each subset using different workers.  Such techniques benefit from not demanding space on each computer to store the full dataset. Generally, one needs to avoid frequent communication
between workers, as it is  time consuming.  In a typical divide and conquer strategy  the data is partitioned  into non-overlapping sub-sets, called {\it shards}, and  each shard is analyzed by a different worker. For such strategies some essential MCMC-related  questions are: 1) how to define the  sub-posterior distributions  for each shard,  and  2) how to combine  the MCMC samples  obtained  from each sub-posterior so that we can recover the same information that would have been obtained  by sampling the full posterior distribution. Existing communication-free parallel methods proposed by \cite{cons}, \cite{emb} and \cite{weierstrass} have in common the fact that the  product of the unnormalized sub-posteriors is equal to  the unnormalized full posterior distribution, but differ in the strategies used to combine the samples. Specifically, \cite{emb} approximate each sub-posterior using kernel density estimators, while \cite{weierstrass} use the Weierstrass transformation. The popular Consensus Monte Carlo (CMC) method \citep{cons}   relies on a weighted averaging approach to combine sub-posterior samples. The CMC relies on theoretical derivations that guarantee its validity   when the full-data posterior and all sub-posteriors  are  Gaussian or mixtures of Gaussian.

We introduce a new communication-free parallel method, the {\it Likelihood Inflating Sampling Algorithm (LISA)}, that also relies on independent and parallel  processing of the shards by different workers to sample the sub-posterior distributions.  The latter are defined differently than in the competing approaches described above. In this paper,  we develop techniques to combine the sub-posterior draws obtained for  LISA  in the  case of   Bayesian Additive Regression Trees (BART) \citep{cart, bart, bartmachine} and compare the  performance of our method with CMC.  

Sections \ref{motiv} and \ref{newmethod} contain a brief review of the CMC algorithm and the detailed description of LISA, respectively. Section \ref{motex} illustrates the potential difference brought by LISA over CMC in a simple Bernoulli example, and includes a simple application of LISA to linear regression models. Section \ref{bart_sec} contains the justification for a modified and improved version of LISA for BART. Numerical experiments and the analysis of socio-economic data presented in Section \ref{sec:exp} examine the computational performance of the algorithms proposed here and compare it with CMC.    We end the paper with some ideas for future work. The Appendix contains  theoretical derivations and descriptions of the steps used when running BART.

\section{Review of Consensus Monte Carlo}\label{motiv}
In this paper we assume that of interest is to generate samples from $\pi(\theta| \vec Y_{N})$, the posterior distribution 
%in order to compute  $$I =\int h(\theta)\pi(\theta|\vec y_{N}) d\theta,$$
%where $\pi(\theta|\vec y_{N})$ is the posterior distribution of 
$\theta$ given the iid sample  $ \vec Y_{N}=\{Y_{1},\ldots,Y_{N}\}$  of size $N$. The assumption is that $N$  is large enough to prohibit running a standard MCMC algorithm in which draws from $\pi$ are obtained on a single computer. We use the  notation $\pi(\theta| \vec Y_{N}) \propto f( \vec Y_{N}|\theta)p(\theta)$, where $f( \vec Y_{N}|\theta)$ is the likelihood function corresponding to the observed data $ \vec Y_{N}$ and $p(\theta)$ is the prior. Major issues with MCMC posterior sampling for big data can be triggered because  a) the data sample is too large to be stored on a single computer, or b) each chain  update  is too costly, e.g. if $\pi$ is sampled via a  Metropolis-Hastings type of algorithm each update requires  $N$ likelihood calculations. 

In order to reduce the computational costs, the  CMC method of \cite{cons}  partitions  the sample into $K$ batches (i.e. $\vec Y_{N}=\cup_{j=1}^{K} Y^{(j)}$) and uses the workers  independently and in parallel  to sample each sub-posterior. More precisely, the $j$-th worker ($j=1,...,K$) will generate samples from the $j$-th sub-posterior distribution defined as:
$$ \pi_{j,CMC}(\theta| Y^{(j)})\propto f( Y^{(j)}|\theta) p(\theta)^{1/K}.$$
Note that the prior for each batch is considered to be $p_j(\theta)=[p(\theta)]^{1/K}$ such that $p(\theta) = \prod_{j=1}^K p_j(\theta)$ and thus the overall full-data unnormalized posterior distribution which we denote as $\pi_{Full}(\theta| \vec Y_{N})$  is equal to  the  product of unnormalized sub-posterior distributions, i.e. $$\pi_{Full}(\theta| \vec Y_{N}) \propto \prod_{j=1}^{K} \pi_{j,CMC}(\theta| Y^{(j)}).$$ 
When the full posterior is Gaussian, the weighted averages of the sub-samples from all batches can be used as full-data posterior draws. That is, assuming $\theta^{(k)}_1,..., \theta^{(k)}_S$ are $S$ sub-samples from the $k$th worker  then the $s$-th approximate full posterior draw will be: $$\theta_s=(\sum_k{w_k})^{-1}\sum_{k}{w_{k}\theta^{(k)}_s}$$ where the weights $w_k=\Sigma^{-1}_k$  are optimal for Gaussian models with $\Sigma_k=\Var(\theta| y^{(k)})$.  
%Recalling the proof, consider the case where $K=2$ with sub-posterior distributions $\pi_{1,Cons}(\theta|\vec y^{(1)}) \sim N(\mu_1,\Sigma_1)$ and $\pi_{2,Cons}(\theta|\vec y^{(2)}) \sim N(\mu_2,\Sigma_2)$. Since $\pi_{Full} \propto \pi_{1,Cons} \times \pi_{2,Cons}$ and the product of two normal distributions is also normal, thus $\pi_{Full}(\theta|\vec y_{N}) \sim N(\mu^*,\Sigma^*)$ where ${\Sigma^*}^{-1}={\Sigma_1}^{-1}+{\Sigma_2}^{-1}$ and $\mu^* = \Sigma^*({\Sigma_1}^{-1}\mu_1 + {\Sigma_2}^{-1}\mu_2)$. Hence drawing $\theta_1$ from $N(\mu_1,\Sigma_1)$ and $\theta_2$ from $N(\mu_2,\Sigma_2)$ will result in $\Sigma^*({\Sigma_1}^{-1}\theta_1 + {\Sigma_2}^{-1}\theta_2) \sim N(\mu^*,\Sigma^*)$.

%In Section 4 we show for a simple model that the diminished prior $p(\theta)^{{1\over K}}$  can lead to loss of efficiency. 
In the next section we introduce an alternative method to define the sub-posteriors in each batch.

\section{Likelihood Inflating Sampling Algorithm (LISA)}\label{newmethod}
LISA  is an alternative to CMC that  also benefits from the random partition of the dataset followed by independently processing  each batch on a different worker. Assuming that the data have been divided into $K$ batches of approximately equal size $n$, we define the sub-posterior distributions for each machine by adjusting the likelihood function without making changes to the prior. Thus the $j$-th sub-posterior distribution will be: $$\pi_{j,LISA}(\theta| Y^{(j)})  \propto {\big[f( Y^{(j)}| \theta)\big]}^{K} {p(\theta)}.$$
Since the data are assumed to be iid, inflating the likelihood function $K$-times is intuitive because the sub-posterior  from each batch of data will be a closer representation of the whole data posterior. We  expect that sub-posteriors sampled by each worker will be closer to the full posterior  thus improving the computational efficiency.

We indeed prove in a theorem below that under mild conditions, LISA's sub-posterior distributions are asymptotically closer to the full posterior than those produced by the CMC-type approach. 

The  Taylor's series expansion for a log-posterior density  $\log \pi(\theta| \vec Y_{N})$ around its posterior mode $\hat{\theta}_N$ yields  the approximation
$$ \log \pi(\theta| \vec Y_{N}) \approx \log \pi(\hat{\theta}_N| \vec Y_{N}) -\frac{1}{2}(\theta-\hat{\theta}_N)^{T}\hat{I}_N(\theta-\hat{\theta}_N)$$ where $\hat{I}_N = - \frac{\partial^2 \log(\pi(\theta| \vec Y_{N}))}{\partial \theta \partial \theta^T}|_{\theta=\hat{\theta}_N}$. 
%is the observed Fisher Information matrix \textcolor{red}{This is not true as there is a prior-related term in there as well. You need to explain why it vanishes (you may miss a 1/n) factor}. 
Exponentiating both sides will result in
$$\pi(\theta| \vec Y_{N}) \approx \pi(\hat{\theta}_N| \vec Y_{N}) \exp \left [-\frac{1}{2}(\theta-\hat{\theta}_N)^{T}\hat{I}_N(\theta-\hat{\theta}_N)\right ]$$ which shows asymptotic normality, i.e. ${\hat{I}_N}^{1/2} (\Theta - \hat{\theta}_N) \xrightarrow{D} N(0,I)$ as $N \rightarrow \infty$ where
{$\Theta \sim \pi(.| \vec Y_{N})$}.  
{Let ${\hat{\theta}_{n,L}}^{(j)}$ and ${\hat{\theta}_{n,C}}^{(j)}$ denote the $j$-$th$ sub-posterior modes in LISA and CMC, respectively. Similarly, ${\hat{I}^{(j)}_{n,L}}$ and ${\hat{I}^{(j)}_{n,C}}$ denote the negative second derivative of the $j$-$th$ log sub-posterior for LISA and CMC, respectively, when calculated at the mode.} Then consider the assumptions,
\begin{enumerate}[label=\alph*)]
\item[{\bf A1}:] There exist $\theta_L, \theta_C$ such that if we define $\epsilon_{n,L}^{(j)}=|\hat \theta_{n,L}^{(j)} - \theta_{L}|$ and $\epsilon_{n,C}^{(j)}=|\hat \theta_{n,C}^{(j)} - \theta_{C}|$, then $ \max \limits _{1\le j \le K} \epsilon_{n,L}^{(j)} \rightarrow 0$ and $ \max \limits _{1\le j \le K} \epsilon_{n,C}^{(j)} \rightarrow 0$ w.p. 1 as $n \rightarrow \infty$.
{\item[{\bf A2}:] $| {\hat{I}^{(i)}_{n,L}} - {\hat{I}^{(j)}_{n,L}} | \longrightarrow 0$ and $| {\hat{I}^{(i)}_{n,C}} - {\hat{I}^{(j)}_{n,C}} | \rightarrow 0$ w.p. 1 $~\forall ~i ~\neq ~j$ as $n \rightarrow \infty$.}
\item[{\bf A3}:] $\pi_{Full}$, $\pi_{j,LISA}$, and $\pi_{j,CMC}$ are unimodal distributions that  have continuous derivatives of order 2.  
\end{enumerate}

\begin{theorem}
Assume  that assumptions
 {\bf A1} {through} {\bf A3}  hold and if $\Theta_{Full} \sim \pi_{Full}(.| \vec Y_{N})$ we also assume ${\hat{I}_N}^{1/2} (\Theta_{Full} - \hat{\theta}_N) \xrightarrow{D} N(0,I)$ as $ N \rightarrow \infty$. If $\Theta_{j,LISA} \sim \pi_{j,LISA}(.| Y^{(j)})$ and $\Theta_{j,CMC} \sim \pi_{j,CMC}(.| Y^{(j)})$ then as $N \rightarrow \infty$ 
$$
{\hat{I}_N}^{1/2} (\Theta_{j,LISA} - \hat{\theta}_N) \xrightarrow{D} N(0,I) \\
~~~\mbox{and}~~~ {\hat{I}_N}^{1/2} (\Theta_{j,CMC} - \hat{\theta}_N) \xrightarrow{D} N(0,K I) ~~~\forall~~ j \in\{1,\ldots,K\}.
$$ 
\label{thm1}
\end{theorem} 
%{\em Proof.} 
\begin{proof} See Appendix. \end{proof}

Theorem 1 shows the difference between sub-posterior distributions for CMC and LISA, with LISA's sub-posterior distributions being asymptotically similar to the full posterior distribution.  This suggests that draws from LISA sub-posteriors can be combined using uniform weights. 
%\textcolor{red}{This specifies that in special circumstances (i.e. sufficient sample size in each batch), using only one batch in LISA can be enough to approximate the full posterior, while in CMC there is need for a further aggregation step given its over-dispersed sub-posteriors, which may be computationally inefficient.}
%Hence, we expect that LISA's batch-specific sub-posteriors will be  better approximations of the full posterior  than the ones generated using the  CMC's design.
%Thus, one straightforward  
%\textcolor{red}{Thus in general, based on Theorem 1, we propose an emerging} strategy for combining LISA's sub-samples, \textcolor{red}{that is} to uniformly sample from the sub-samples produced by each worker. 

{\it Remarks:} 
\begin{enumerate}
\item When  data are iid we expect the shards to become more and more similar as $N$ (and thus $n=N/K$) increases and  assumption {\bf A1} is  expected to hold for general models.

\item {Assumption {\bf A2} in Theorem 1 holds due to the structural form of sub-posteriors in LISA and CMC.}

\item The validity of using uniform weights with LISA's sub-posterior draws is justified asymptotically, but we will see that this approximation can be exact in some examples, e.g. for a Bernoulli model with balanced batch samples, while  in others modified weights can improve the performance of the sampler. In this respect LISA is similar to other embarrassing parallel strategies where one must carefully consider the model of interest in order to find the best way to combine the sub-posterior samples. 
 
\end{enumerate}

In the next section we will illustrate LISA in some simple examples and compare its performance to the full-data posterior sampling as well as CMC.

\section{Motivating Examples}\label{motex}

In this section we examine some simple examples where theoretical derivations can be carried out in detail. We emphasize the difference between LISA and CMC. 

\subsection{Bernoulli Random Variables}\label{simple}
Consider $y_1,...,y_N$ to be N i.i.d. Bernoulli random variables with parameter $\theta$. Hence,  we consider a prior  $p(\theta)=\mbox{Beta}(a,b)$. Assuming that we know little about the size of $\theta$ we  set $a=b=1$ which corresponds to a $U(0,1)$ prior. The resulting full-data posterior  $\pi_{Full}(\theta| \vec Y_{N})$ is  Beta$(S+a,N-S+b)$ where $S=\sum_{i=1}^{N} y_{i}$ is the total number of ones. Suppose we divide the data into $K$ batches with $S_j$ number of ones in batch $j$, such that $S_j=\frac{S}{K}$ $~\forall~ j \in \{1,..,K\}$, i.e. the number of 1's are divided equally between batches. Then the $j^{th}$ sub-posterior based on batch-data of size $n=\frac{N}{K}$ for each method will be:
\begin{itemize}
\item \textbf{CMC:} 
\begin{align*}
\pi_{j,CMC}(\theta| Y^{(j)}) = \mbox{Beta}\left ({S_j}+{a-1 \over K} +1, {n -S_j}+{b-1 \over K}+1\right )\\
 = \mbox{Beta}\left ({S \over K}+{a-1 \over K} +1, {N -S \over K}+{b-1 \over K}+1\right )
 \end{align*}
\item \textbf{LISA:} 
\begin{align*} \pi_{j,LISA}(\theta| Y^{(j)}) = \mbox{Beta}(S_{j}K+a, (n-S_{j})K+b) \\
 = \mbox{Beta}(S+a,N-S+b)~~~~~~~~~
\end{align*} which implies
$$ \pi_{j,LISA}(\theta| Y^{(j)})  = \pi_{Full}(\theta| \vec Y_{N}) ~~~\forall ~j \in\{1,...,K\}.$$
\end{itemize}
In this simple case any one of  LISA 's sub-posterior distributions is equal to the full posterior distribution if the batches are balanced, i.e. the  number of 1's are equally split across all batches. Thus, LISA's sub-samples from any batch will represent correctly  the full posterior. On the other hand,  the draws from the CMC sub-posterior distributions will need to be recombined to obtain a representative sample from the true full posterior $\pi_{Full}(\theta| \vec Y_{N})$.

However, when the number of ones is unequally distributed among the batches it is not easy to pick the winner between CMC and LISA as both require  a careful weighting of each batch sub-posterior samples.  

In the remaining part of this paper, we will mainly focus on the performance of LISA when it is applied to  the Bayesian Additive Regression Trees (BART) model.  Interestingly, we discover that using a minor modification inspired by running LISA on the simpler  Bayesian Linear Regression model we can approximate the full posterior. The idea behind the modification is described in the next section.

\subsection{\textbf{Bayesian Linear Regression}} \label{linear}
Consider a standard linear regression model 
\begin{equation}  Y = X  \beta +  \epsilon 
\label{lrm} 
\end{equation} where $ \beta \in \RR^{p}$, $X \in \RR^{N \times p}$ and $Y, \epsilon \in \RR^{N}$ with $ \epsilon \sim \mathcal{N}_N(0,\sigma^2 \I_N)$. 
%
%We consider the conjugate priors
%\begin{align*} 
%p(\beta|\sigma^2) \sim \mathcal{N}_p(\mu_0,\sigma^2 {\Omega_0}^{-1}),~~~~ \\ 
%p(\sigma^2) \sim \mbox{Inv-Gamma}(a_0,b_0).
%\end{align*}
To simplify the presentation we consider the improper prior
\beq
p(\beta,\sigma^2) \propto \sigma^{-2}.
\label{prior}
\eeq

Straightforward calculations show that the conditional posterior distributions for the full data are
\beqn
\pi_{Full}(\sigma^2| Y,X) &=& \mbox{Inv-Gamma}\left({N-p \over 2}, {s^2(N-p) \over 2}\right)
\label{sigma-post} \\
\pi_{Full}(\beta|\sigma^2,Y, X) &=& N\big(\hat \beta, \sigma^2 (X^TX)^{-1}\big) 
\label{beta-post}
\eeqn
where $\hat \beta=(X^TX)^{-1}X^TY$ and $s^2 = {(Y-X \hat \beta)^T(Y-X \hat \beta) \over N-p}$.

An MCMC sampler designed to sample from $\pi_{Full}(\beta,\sigma^2|Y,X)$ will iteratively sample $\sigma^2$ using \eqref{sigma-post} and then $\beta$ via \eqref{beta-post}. If we denote $\beta_{Full}$ the r.v. with  density $\pi_{Full}(\beta|Y,X)$ then, using the iterative formulas for conditional mean and variance we obtain
$$E[\beta_{Full}|Y,X]=(X^TX)^{-1}X^TY$$
and
\beq
\Var(\beta_{Full}|Y,X)= {{(X^TX)^{-1}}} {(N-p)/2 \over (N-p)/2{-}1} s^2  =  {{(X^TX)^{-1}}} s^2+O(N^{-1}).
\label{var-full}
\eeq

%%%%%%% Previously it was:
%\beq
%\Var(\beta_{Full}|Y,X)= s^2 {(N-p)/2 \over (N-p)/2+1} = s^2+O(N^{-1}).
%\label{var-full}
%\eeq

We examine below the statistical properties of the samples produced by LISA. If the data are divided into K equal batches of  size $n=N/K$, let us denote $Y^{(j)}$ and $X^{(j)}$ the  response vector and model matrix from the $j$th batch, respectively.

With the prior given in \eqref{prior}, the sub-posteriors produced by LISA have the following conditional densities
\beqn
\pi_j(\sigma^{2}| Y^{(j)},X^{(j)}) &=& \mbox{Inv-Gamma}\left({N-p \over 2}, {Ks_j^2(n-p) \over 2}\right)
\label{sigma-post-lisa}\\
\pi_j(\beta|\sigma^2,Y^{(j)}, X^{(j)}) &=& N\big(\hat \beta_j, {\sigma^2 \over K} (X^{(j)\;T} X^{(j)})^{-1} \big), 
\label{beta-post-lisa}
\eeqn
where  $\hat \beta_j=(X^{(j)\; T}X^{(j)})^{-1}X^{(j)\; T}Y^{(j)}$ and $s_j^2 = {(Y^{(j)}-X^{(j)} \hat \beta_j)^T(Y^{(j)}-X^{(j)} \hat \beta_j) \over n-p}$ for all $1\le j \le K$.

A simple Gibbs sampler designed to sample from $\pi_j(\beta,\sigma^2|Y^{(j)},X^{(j)})$ will iteratively sample $\sigma^2$ from \eqref{sigma-post-lisa} and then $\beta$ from \eqref{beta-post-lisa}.
%\beqn
%\sigma^2 \sim \mbox{Inv-Gamma}\left({N-p \over 2}, {Ks_j^2(n-p) \over 2}\right) \nonumber\\
%\beta \sim N(\hat \beta_j, {\sigma^2 \over K} (X^{(j)\;T} X^{(j)})^{-1}). \nonumber
%\eeqn

It can be shown using the iterative formulas for conditional means and variances that 
$$E[\beta| Y^{(j)}, X^{(j)}] =\hat \beta_j$$ and
$$
\Var(\beta|Y^{(j)}, X^{(j)})= (X^{(j)\;T} X^{(j)})^{-1} {s_j^2(n-p)/2 \over (N-p)/2-1} =(X^{(j)\;T} X^{(j)})^{-1} {s_j^2(n-p) \over (N-p)} +O(N^{-1}).
$$

In order to combine the sub-posterior  samples  we propose using the weighted average 
\beq
\beta_{LISA} = (\sum_{j=1}^K W_j)^{-1} \sum_{j=1}^K W_j \beta_j,
\label{combo-lisa}
\eeq
where $\beta_j \sim \pi_j(\beta| Y^{(j)}, X^{(j)})$ and $W_j={X^{(j)\;T} X^{(j)} \over \sigma^2}$. Since $\sum_{j=1}^K X^{(j)\;T} X^{(j)} =X^TX$ we get 
\beq 
E[\beta_{LISA}|Y,X] = \hat \beta =(X^T X)^{-1} X^TY
\eeq
and
\beq
\Var(\beta_{LISA}|Y,X) =  (X^TX)^{-1} {n-p \over N-p} \left [\sum_{j=1}^K s_j^2 (X^{(j)\;T} X^{(j)})\right ] (X^TX)^{-1} \approx (X^TX)^{-1} {n-p \over N-p} s^2,
\label{var-lisa}
\eeq 
where the last approximation in \eqref{var-lisa} is based on the assumption that $s_j^2 \approx s^2$ as both are unbiased estimators for $\sigma^2$ based on $n$ and, respectively, $N$ observations.  It is apparent that the variance computed in \eqref{var-lisa} is roughly $K$ times smaller than the target given in \eqref{var-full}.  In order to avoid underestimating the variance of the posterior distribution we propose a  modified LISA sampling algorithm which consists of the following steps:
\beqn
\sigma^2 &\sim& \mbox{Inv-Gamma}\left({N-p \over 2}, {Ks_j^2(n-p) \over 2}\right) \nonumber\\
\tilde \sigma&=& \sqrt{K} \sigma \nonumber \\
\tilde \beta &\sim& N(\hat \beta_j, {\tilde \sigma^2 \over K} (X^{(j)\;T} X^{(j)})^{-1})= N(\hat \beta_j,  \sigma^2 (X^{(j)\;T} X^{(j)})^{-1}). \nonumber
\eeqn
The intermediate step simply adjusts the variance samples so that 
$$\Var(\tilde \beta|Y^{(j)}, X^{(j)}) =(X^{(j)\;T} X^{(j)})^{-1} {s_j^2 K(n-p)/2 \over (N-p)/2-1} = (X^{(j)\;T} X^{(j)})^{-1} {s_j^2 K(n-p) \over (N-p)} + O(N^{-1}).$$
In turn, if we define
\beq
\beta_{modLISA} = (\sum_{j=1}^K W_j)^{-1} \sum_{j=1}^K W_j \tilde \beta_j,
\label{combo-modlisa}
\eeq
 then 
$E[\beta_{modLISA} |Y,X] = (X^TX)^{-1} X^TY$
and 
 \beq
\Var(\beta_{modLISA} | Y,X) =  (X^TX)^{-1} {K(n-p) \over N-p} \left [\sum_{j=1}^K s_j^2 (X^{(j)\;T} X^{(j)})\right ] (X^TX)^{-1} \approx (X^TX)^{-1} {K(n-p) \over N-p} s^2.
\label{var-modlisa}
\eeq 
Note that when the regression has only an intercept, i.e. $X$ consists of a column of 1's, the 
weights $W_j \propto (\sigma^2)^{-1}$.

While both  \eqref{combo-modlisa}  and \eqref{combo-lisa} produce samples that have the correct mean, from  equations (\ref{var-full}), (\ref{var-lisa}) and (\ref{var-modlisa}) we can see that   the weighted average of the modified LISA samples have the variance closer to the desired target.  

In the next section, we will examine LISA's performance on a more complex model, the Bayesian Additive Regression Trees (BART). The discussion above will guide our construction of a modified version of LISA for  BART.

\section{\textbf{Bayesian Additive Regression Trees (BART)}}\label{bart_sec}
Consider the nonparametric regression model: $$y_i = f(x_i) + \epsilon_i,~~~~ \epsilon_i \sim \mathcal{N}(0,\sigma^{2})  ~~~i.i.d.$$
where $x_i=(x_{i1},...,x_{ip})$ is a $p$-dimensional vector of inputs and $f$ is approximated by a sum of $m$ regression trees:$$f(x) \approx \sum_{j=1}^{m}{g(x;T_{j},M_{j})}$$
where $T_{j}$ denotes a binary tree consisting of a set of interior node decision rules and a set of terminal nodes. $M_{j} = \{ \mu_{1j},..., \mu_{bj} \} $ is the set of parameter values associated with the $b$ terminal nodes of $T_{j}$. In addition, $g(x;T_{j},M_{j})$ is the function that maps each $x$ to a $\mu_{ij} \in M_{j}$. Thus the regression model is approximated by a sum-of-trees model
$$y_i = \sum_{j=1}^{m}{g(x_i;T_{j},M_{j})} + \epsilon_i~,~~~~ \epsilon_i \stackrel{iid}{\sim} \mathcal{N}(0,\sigma^{2}) $$
Let $ {\theta} := ((T_{1},M_{1}),...,(T_{m},M_{m}),\sigma^{2}) $ denote the vector of model parameters. Below, we  briefly describe the prior specifications stated in \cite{bart} and \cite{cart}.\\

\noindent {\large \textbf{Prior Specifications:}}

\begin{itemize}
\item Prior Independence and Symmetry:
$$p((T_{1},M_{1}),...,(T_{m},M_{m}),\sigma) = \Big[ \prod_{j}{p(M_{j}|T_{j})p(T_{j})} \Big] p(\sigma)$$

where $ p(M_{j}|T_{j}) = \prod_{i}{p(\mu_{ij}|T_{j})}$.
\item Recommended number of trees: m=200 \citep{bart} and m=50 \citep{bartmachine}
\item Tree prior $p(T_{j})$, is characterised by three aspects: \begin{enumerate}
\item The probability that a node at depth $d=0,1,...$ is non-terminal, which is assumed to have the form $\alpha{(1+d)}^{-\beta}$, where $\alpha \in (0,1)$~ and $\beta \geq 0$. (recommended values are $\alpha=0.95$ and $\beta=2$)
\item The distribution on the splitting variable assignments at each interior node which is recommended to have a uniform distribution.
\item The distribution on the splitting rule assignment in each interior node, conditional on the splitting variable which is also recommended to have a uniform distribution.
\end{enumerate}
\item The conditional prior for $\mu_{ij}$ is  $\mathcal{N}(\mu_{\mu}, \sigma_\mu^2)$ such that:
$$ \left\{ 
  \begin{array}{ l}
    m\mu_{\mu} - k\sqrt{m}\sigma_{\mu} = y_{min}\\
    m\mu_{\mu} + k\sqrt{m}\sigma_{\mu} = y_{max}
  \end{array} \right.$$
  with $k=2$ recommended.
\item The prior for $\sigma^{2}$ is $\mbox{Inv-Gamma}(\frac{\nu}{2},\frac{\nu\lambda}{2})$ where $\nu =3$ is recommended and $\lambda$ is chosen such that $p( \sigma < \hat{\sigma}) = q$ with recommended $q=0.9$ and sample variance $\hat{\sigma}$. 
\end{itemize}
Hence the posterior distribution will have the form:
\begin{multline}
\pi({\theta})=\pi({\theta}|Y,X) \propto \underbrace{\bigg \{ {({\sigma}^{2})}^{-\frac{n}{2}} e^{-\frac{1}{2{\sigma}^{2}} \sum_{i=1}^{n}{{(y_{i} - \sum_{j=1}^{m}{g(x_{i};M_{j},T_{j})} )}^{2} } } \bigg\} }_\text{Likelihood} \times \\  \underbrace{\bigg\{  \underbrace{{({\sigma}^{2})}^{-\frac{\nu}{2}-1} e^{-\frac{\nu \lambda}{2{\sigma}^{2}}} }_\text{Prior of $\sigma^{2}$}\Big[ \prod_{j=1}^{m}{ \sigma_\mu^{-b_{j}} {(2\pi)}^{-\frac{b_{j}}{2}} e^{-\frac{1}{2\sigma_\mu^{2}} \sum_{k=1}^{b_{j}}{{(\mu_{kj} - \mu_{\mu})}^{2} }}} p(T_{j})  \Big] \bigg\}}_\text{Prior}.
\end{multline}
Gibbs Sampling is used to sample from this posterior distribution. The algorithm iterates between the following steps:
\begin{itemize}
\item ${\sigma}^{2} ~| ~(T_{1},M_{1}),...,(T_{m},M_{m}),Y,X  ~  \propto \mbox{Inv-Gamma} (\rho,\gamma)$ \\
where $\rho = \frac{\nu+n}{2} $ and $\gamma = \frac{1}{2}~[~\sum_{i=1}^{n}{{(y_{i} - \sum_{j=1}^{m}{g(x_{i};M_{j},T_{j})} )}^{2} } + \lambda \nu~]$.\\
\item $(T_{j},M_{j}) ~| ~T_{(j)}, M_{(j)},\sigma,Y,X$ which is the same as drawing from the conditional  $(T_{j},M_{j})~ | ~R_{j},\sigma$
where $T_{(j)}$ denotes all trees except the $j$-th tree, and residual $R_{j}$ is defined as:  $$R_{j} =~ g(x;,M_{j},T_{j}) + \epsilon =~ y - \sum_{k \neq j}{g(x;M_{k},T_{k}}) .$$
The sampling of $(T_j,M_j)$ is performed in two steps:
\begin{enumerate}
\item $T_{j}~|~R_{j},\sigma$ and
\item $M_{j}~|~T_{j},R_{j},\sigma$.
\end{enumerate}

Step 2 involves sampling from each component of $M_{j}$ using 
$$\mu_{ij}~|~T_{j},R_{j},\sigma ~\sim~ \mathcal{N}\left(\frac{\frac{\sigma^2}{\sigma_\mu^2}~\mu_\mu ~+~ n_{i}\bar{R}_{j(i)}}{\frac{\sigma^2}{\sigma_\mu^2}~+~n_{i}},\frac{\sigma^2}{\frac{\sigma^2}{\sigma_\mu^2}~+~n_{i}}\right)$$ 
where $\bar{R}_{j(i)}$ denotes the average residual (computed without tree $j$) at terminal node $i$ with total number of observations $n_i$.
The conditional density of $T_{j}$ in step  1  can be expressed as:
\begin{equation} \label{eq:1}
p(T_{j}~|~R_{j},\sigma)~~ \propto ~~ p(T_{j}) \int{p(R_{j}~|~M_{j},T_{j},\sigma)~p(M_{j}~|~T_{j},\sigma)~d{M}_{j}}.
\end{equation}
\end{itemize}
The Metropolis-Hastings (MH) algorithm is then applied to draw $T_{j}$ from (\ref{eq:1}) with four different proposal moves on trees:
\begin{itemize}
\item \textbf{GROW:} growing a terminal node (with probability 0.25);
\item \textbf{PRUNE:} pruning a pair of terminal nodes (with probability 0.25);
\item \textbf{CHANGE:} changing a non-terminal rule (with probability 0.4) \citep[][ change rules only for parent nodes with terminal children]{bartmachine};
\item \textbf{SWAP:} swapping a rule between parent and child (with probability 0.1)  \citep[This proposal move was removed by ][]{bartmachine}.
\end{itemize}
\noindent Detailed derivations involving the Metropolis-Hastings acceptance ratios are described in the Appendix. 

Two existing packages in R, "BayesTree" and "bartMachine", can be used to run BART on any dataset, but as the sample size increases, these packages tend to run slower. In these situations we expect methods such as LISA or CMC to become useful, and for a fair illustration of the advantages gained we have used our own R implementation of BART and applied the same structure to implement LISA and CMC algorithm for BART. The Metropolis-Hastings acceptance ratios for LISA and CMC are also reported in the Appendix. 

As discussed by \cite{cons}, the approximation to the posterior produced by the CMC algorithm  can be poor. Thus, for comparison reasons, we applied both LISA and CMC to BART using a simulated dataset (described further) with $K=30$ batches. Given Theorem 1, since LISA's sub-posterior distributions are asymptotically equivalent to the full posterior distribution, we examined its performance by uniformly taking sub-samples from all its batches as an approximation to full posterior samples. We will see further that LISA with uniform weights produces higher prediction accuracy compared to CMC. However, they both perform poorly in approximating the posterior samples as they generate larger trees and under-estimate $\sigma^2$, which results in over-dispersed posterior distributions.  

The following sub-section discusses a modified version of LISA for BART which will have significant improvement in performance. 

\subsection{\textbf{Modified LISA for BART}} 

The under estimation of $\sigma^{2}$ when applying LISA to BART is similar to the problem encountered when using LISA for the linear regression model discussed in Section \ref{linear}. This is not a coincidence since 
BART is also a linear regression model, albeit one where the set of  independent variables is determined through a highly sophisticated process.  We will show below that when applying a similar variance adjustment to the one {discussed in Section \ref{linear}}, the Modified  LISA (modLISA) for BART will exhibit superior computational and statistical efficiency compared to either LISA or CMC. 

Just like in the regression model we ``correct''  the sampling algorithm by adjusting the residual variance.  We start with the conditional distribution of tree $j$ from expression \eqref{eq:1} which takes the form $$p(T_{j}~|~R_{j},\sigma)~~ \propto ~~ p(T_{j}) \int{p(R_{j}~|~M_{j},T_{j},\sigma)~p(M_{j}~|~T_{j},\sigma)~d{M}_{j}}.$$ Note that only the conditional distribution of the residuals, $R_{j}~|~M_{j},T_{j},\sigma$ is affected by the modifications brought by LISA. The Metropolis-Hastings acceptance ratios for tree proposals contain  three parts: the transition ratio, the likelihood ratio and the tree structure ratio. The modifications brought by LISA will influence only  the  likelihood ratio which is constructed from the conditional distributions of residuals. Consider the likelihood ratio for GROW proposal in LISA  (full details are presented in the  Appendix)
\begin{multline}\frac{P(R~|~T_{*},\sigma^{2})}{P(R~|~T,\sigma^{2})}=
\sqrt{\frac{\sigma^{2}(\sigma^{2}+{K}n_{l}\sigma_\mu^{2})}{(\sigma^{2}+{K}n_{l_{L}}\sigma_\mu^{2})(\sigma^{2}+{K}n_{l_{R}}\sigma_\mu^{2})}}~\times~\\\\ \exp\bigg\{ \frac{{K^2} \sigma_\mu^{2}}{2\sigma^{2}}\bigg[ \frac{(\sum_{i=1}^{n_{l_{L}}}{R_{l_{L},i}})^{2}}{\sigma^{2}+{K}n_{l_{L}}\sigma_\mu^{2}}+\frac{(\sum_{i=1}^{n_{l_{R}}}{R_{l_{R},i}})^{2}}{\sigma^{2}+{K}n_{l_{R}}\sigma_\mu^{2}} - \frac{(\sum_{i=1}^{n_{l}}{R_{l,i}})^{2}}{\sigma^{2}+{K}n_{l}\sigma_\mu^{2}}   \bigg]  \bigg\}
\label{eq:modif}
\end{multline}
where $n_{l}$ is the total number of observations from batch-data that end up in terminal node $l$. The newly grown tree, $T_*$, splits terminal node $l$ into two terminal nodes (children) $l_L$ and $l_R$, which will also divide $n_{l}$ to $n_{l_L}$ and $n_{l_R}$ which are the corresponding number of observations in each new terminal node. By factoring out $K$ in \eqref{eq:modif}, we can rewrite it as
\begin{multline}\frac{P(R~|~T_{*},\sigma^{2})}{P(R~|~T,\sigma^{2})}=
\sqrt{\frac{\frac{\sigma^{2}}{K}(\frac{\sigma^{2}}{K}+n_{l}\sigma_\mu^{2})}{(\frac{\sigma^{2}}{K}+n_{l_{L}}\sigma_\mu^{2})(\frac{\sigma^{2}}{K}+n_{l_{R}}\sigma_\mu^{2})}}~\times~\\\\ 
\exp\bigg\{ \frac{\sigma_\mu^{2}}{2\frac{\sigma^{2}}{K}}\bigg[ \frac{(\sum_{i=1}^{n_{l_{L}}}{R_{l_{L},i}})^{2}}{\frac{\sigma^{2}}{K}+n_{l_{L}}\sigma_\mu^{2}}+\frac{(\sum_{i=1}^{n_{l_{R}}}{R_{l_{R},i}})^{2}}{\frac{\sigma^{2}}{K}+n_{l_{R}}\sigma_\mu^{2}} - \frac{(\sum_{i=1}^{n_{l}}{R_{l,i}})^{2}}{\frac{\sigma^{2}}{K}+n_{l}\sigma_\mu^{2}}   \bigg]  \bigg\}.
\label{eq:BatchSingle}
\end{multline}
Expression \eqref{eq:BatchSingle} {shows a similar residual variance that is $K$ times smaller in each batch, and hence following the discussion in Section \ref{linear},
%in LISA is equivalent to BatchSingleMachine except for the smaller variance considered for the conditional distribution of residuals ($\frac{\sigma^2}{K}$), while in BatchSingleMachine each residual has conditional distribution given as $R_{.j}~|~M_{j},T_{j},\sigma \sim \mathcal{N}(g(.;M_j,T_j),\sigma^2)$. 
%Hence, 
to achieve similar variance,} 
% for residuals as in BatchSingleMachine, 
we need to modify LISA for BART by {adding the intermediate step} $\tilde \sigma^2 = K \sigma^2$ when updating \textit{trees}  {in each batch,} and then taking a weighted average combination of sub-samples (similar to Bayesian linear regression).  As in Section \ref{linear},   we don't apply any changes when updating $\sigma^2$.
%Note that in modLISA, we don't apply any changes in updating $\sigma^2$, i.e. we keep the same conditional distribution as in LISA:
%\begin{gather*}
%{\sigma}^{2} ~| ~(T_{1},M_{1}),...,(T_{m},M_{m}),y_k,X_k  ~~ \propto \mbox{Inv-Gamma}  (\rho,\gamma) \end{gather*}
%where $\rho = \frac{\nu+{K}n}{2}$  and $\gamma = \frac{1}{2}~[~{K}\sum_{i=1}^{n}{{({y_{i}}^{(k)} - \sum_{j=1}^{m}{g({x_{i}}^{(k)};M_{j},T_{j})} )^{2}} } + \lambda \nu~]$. This is obviously different from the conditional distribution of $\sigma^2$ in BatchSingleMachine where there is no $K$. 
All our numerical experiments show that
%, despite this difference, 
modLISA {also generates accurate predictions in BART, since} the modification corrects the bias in the  posterior draws of $\sigma^2$ and properly calibrates the size of the trees. 

The BART algorithm will split the covariate space into disjoint subsets and on each subset a regression with only an intercept is fitted. Therefore, as suggested by the discussion in 4.2 the weight assigned to each batch will be proportional to the estimate of $\sigma^2$ in that batch. In the following sections we examine the improvement brought by modLISA when compared to LISA and CMC.

\section{Numerical Experiments}
\label{sec:exp}

\subsection{\textbf{The Friedman's function}}\label{simulated}
We have simulated data of  size  $N=20,000$ from Friedman's test function \citep{friedman} $$f(x) = 10\sin(\pi x_{1} x_{2}) + 20(x_{3}-0.5)^{2} + 10 x_{4} + 5x_5,$$ where the  covariates  $x=(x_{1}, \ldots, x_{10})$ are simulated independently from a $U(0,1)$ and $y \sim \mathcal{N}(f(x),\sigma^2)$ with $\sigma^2=9$. Note that five of the ten covariates are unrelated to the response variable.  We have also generated test data containing 5000 cases. We apply BART to this simulated dataset using the default hyperparameters stated in Section \ref{bart_sec} with $m=50$ to generate posterior draws of $(T,M,\sigma^2)$ that, in turn, yield posterior draws for $f(x)$ using the approximation $\hat{f}(x) \approx \sum_{j=1}^{m}g(x;\hat{T_j},\hat{M_j})$ for each  $x=(x_{1}, \ldots, x_{10})$. Since in this case the true $f$ is known, one can compute the root mean squared error (RMSE) using average posterior draws of $\hat{f}(x)$ for each $x$ (i.e. $\overbar{\hat{f}(x)}$), as an estimate to measure its performance, i.e. RMSE $= \sqrt{\frac{1}{N}\sum_{i=1}^{N}{(f(x_i) - \overbar{\hat{f}(x_i)})^2}}$.  It is known that SingleMachine BART may mix poorly when it is run on an extremely large dataset with small residual variance. However since the data simulated is of reasonable  size  and  $\sigma$  is not very small the SingleMachine BART is expected to be a good benchmark for comparison \citep[see discussion in][]{pratola}.

\subsubsection{Comparison of modLISA with Competing Methods}\label{comparison_modLISA}
We have implemented modLISA, LISA, and CMC for BART with $K=30$ batches on the simulated data for 5000 iterations with a total of 1000 posterior draws. Table \ref{modLISA} shows results from all methods including the SingleMachine which runs BART on the full dataset using only one machine. Results are averaged over three different realizations of train and test data, and include  the Train and Test RMSE for each method, along with tree sizes, $\sigma^2$ estimates and their 95\% Credible Intervals (CI).   The summaries presented in Table \ref{modLISA} show that although LISA has better prediction performance than CMC, it does a terrible job at estimating $\sigma^2$, its estimate  being orders of magnitude smaller than the one produced by CMC. 
 CMC and LISA both generate larger trees compared to SingleMachine, with CMC generating trees that are ten times larger  than LISA's. 
 One can see that  modLISA  with weighted averages  dominates both CMC and LISA  across all performance indicators since it yields the smallest RMSE, the smallest tree size, and less biased $\sigma^2$ estimates. Generally, modLISA  generates results that are by far the closest to the ones produced by SingleMachine.

%Overall, it is clear that  neither CMC nor LISA  exhibit  desirable properties for BART. The story changes with  modLISA  with weighted average which dominates both CMC and LISA  across all performance indicators since it has the lowest RMSE, the highest coverage, the lowest tree sizes, and less biased $\sigma^2$ estimates and produces results that are   the closest to the ones produced by SingleMachine.

%In addition, Table \ref{modLISA} also includes the average Train and Test coverage of 95\% prediction intervals, i.e. the proportion of 1000 newly simulated $y$ at a given train or test $x$ that is covered by its corresponding 95\% prediction interval.

%And although LISA has better prediction performance compared to CMC, we also observe much higher coverage probabilities for  CMC.

% 20,000 and K=30 averaged over 3 realizations
\begin{table*}[h]
\footnotesize
\centering
\caption{Comparing Train \& Test RMSE, tree sizes, and average post burn-in $\hat{\sigma}^2$ with 95\% CI in each method for $K=30$ to SingleMachine BART (all results are averaged over three different realizations of data).}
\label{modLISA}
\begin{tabular}{crrrrrc}
\hline \\
Method  & \multicolumn{1}{c}{TrainRMSE} & \multicolumn{1}{c}{TestRMSE} & \multicolumn{1}{c}{Tree Nodes} & \multicolumn{1}{c}{Avg $\hat{\sigma}^2$} & \multicolumn{1}{c}{95\% CI for $\sigma^2$}  \\ \\
\hline
$CMC$ & 2.73 & 2.94  & 602 & 1.91 &  {[1.45~,~2.88]}\\ 
$LISA ~(unif ~ wgh)$ & 1.18 & 1.19  & 55 & 0.001  & {[0.0009~,~0.0011]}\\ 
$modLISA~(wgh~avg)$ & 0.57 & 0.59  & 7 & 7.97 &  {[7.87~,~8.08]} \\ 
$SingleMachine $ & 0.55 & 0.56  & 7 & 9.04  & {[8.85~,~9.21]}\\ 
 \hline
\end{tabular}
\end{table*}

\begin{table*}[h]
\centering
\caption{Average acceptance rates of tree proposal moves.}
\label{acc}
\begin{tabular}{crrrc}
\hline \\
Method & \multicolumn{1}{c}{GROW} & \multicolumn{1}{c}{PRUNE}   & \multicolumn{1}{c}{CHANGE} \\ \\
\hline
$CMC$ &  21\% &   0.03\% &  34\%  \\ 
$LISA$ &  1.8\% &   0.5\% &  1.6\%  \\ 
$modLISA$ & 20\% &   26\% &  19\%  \\ 
$SingleMachine$ &  9\% &  10\% &  6\%  \\  
 \hline
\end{tabular}
\end{table*}

The size of trees produced by each method is in sync with   the average acceptance rates of each tree proposal move  shown in Table \ref{acc}.  It is noticeable the difference between CMC and LISA 's average acceptance rates between growing a tree and pruning one. 
%In addition, this difference is smaller in LISA which is why LISA generates smaller trees compared to CMC ($0.5/1.8 \approx 0.28$ in LISA versus $0.03/21 \approx 0.0014$ in CMC), while still having larger trees compared to SingleMachine. 
On the other hand, modLISA has overall larger acceptance rates with the smallest relative absolute difference between growing and pruning probabilities compared to LISA and CMC ($6/20 = 23.1$\% for modLISA, $98.6$\% for CMC, and  $72.2$\% for LISA) and is closest to SingleMachine ($10$\%). Overall, modLISA induced a significant reduction in tree sizes by preserving a balance between growing and pruning trees which also improves exploring the posterior distribution.

% K=30  &  obs = 20,000 ----COVERAGE
\begin{table*}[h]
\footnotesize
\centering
\caption{Average coverage for  95\% credible intervals constructed for training (TrainCredCov) and test (TestCredCov) data and 95\% prediction intervals constructed for training (TrainPredCov) and test (TestPredCov)  data. The prediction interval coverage is estimated based on 1000 iid samples,  $N=20,000$ and $K=30$. All results are averaged over three different realizations of data.}
\label{CovmodLISA}
{
\begin{tabular}{crrrc}
\hline \\
Method  &  \multicolumn{1}{c}{TrainPredCov} & \multicolumn{1}{c}{TestPredCov} &  \multicolumn{1}{c}{TrainCredCov} & \multicolumn{1}{c}{TestCredCov}  \\ \\
\hline
$CMC$ & 45.71 \% & 47.83 \%  & 81.95  \% & 99.99 \% \\ 
$LISA ~(unif ~ wgh)$ &  1.54 \% & 1.54 \%  & 100 \% & 100 \%\\ 
$modLISA~(wgh~avg)$ &  92.93 \% & 92.91 \%  & 60.88 \% & 58.45 \%\\ 
$SingleMachine $ & 94.67 \% & 94.65 \% & 71.58 \% & 71.54 \%\\ 
 \hline
\end{tabular}
}
\end{table*}

For a more clear comparison of the methods, Table \ref{CovmodLISA} shows the average  coverage of 95\% credible intervals (CI) for predictors $f(x)$ and 95\% prediction intervals (PI) for future responses $y$. The calculations are made for the values of $y$ and $f(x)$ in the training and test data sets.

The coverage for CI is  given by the averaging for all training or test data of 
$$ { \#\{ f(x_i) \in \hat I_{f(x_i)}: \; 1\le i\le N\} \over N }$$
where $\hat I_f(x_i)$ is the CI for $f(x_i)$ estimated based on the MCMC draws from $\pi$.
 
    The coverage of  the PI corresponding to a pair $(y_i,f(x_i))$ is given by the proportion  of 1000 iid samples generated from the true generative model $N(f(x_i),\sigma^2)$ that fall between its limits, i.e. the average over training or test data  of 
    $$ { \# \{\tilde y_j \in \hat J_{y_i}: \; \tilde y_j \stackrel{iid}{\sim} N(f(x_i),\sigma^2) 1\le j \le 1000\} \over 1000},$$ 
    where $\hat J_{y_i}$ is the PI for $y_i$. The PI coverage in modLISA and SingleMachine are very close to nominal and vastly outperform the PI's produced using LISA or CMC. 

One can see that coverages of the CI built via  CMC and LISA are high, which is not surprising since both algorithms produce over-dispersed approximations to the conditional distributions of $f(x)$. Our observation is that the CI for LISA and CMC are too wide to be practically useful. Also, modLISA and SingleMachine have much lower CI coverage than nominal which, as pointed out by one of the referees, is also expected due to the systematic bias induced by the discrepancy between the functional forms of the true predictor (continuous) and of the one fitted by  BART (piecewise constant).  Thus, the  CI for $f(x)$ will exhibit poor coverage as they are centered around a biased estimate of $f(x)$.  

In order to verify that this is indeed the case we  have generated a dataset of size 20,000 from the piecewise constant function:
$$
 f(x) = 
   \one_{[0,0.2)} (x_1)+ 2 \cdot \one_{[0.2,0.4)}(x_1)+3\cdot \one_{[0.4,0.6)}(x_1)+4\cdot \one_{[0.6,0.8)}(x_1)+5 \cdot \one_{[0.8,1)}(x_1) 
  $$
where $\one_{[a,b)}(x)=1$ if $x\in[a,b)$ and 0 otherwise,   $x=(x_{1},\ldots,x_{10}) \in (0,1)^{10}$ is a ten-dimensional input vector, with $x_i \sim \ru(0,1)$, and $y \sim \mathcal{N}(f(x),9)$. Additional 5000 data have also been simulated as test cases. Table \ref{piecewise} summarizes the analysis with $K=30$ and  confirms a sharp decrease in RMSEs even though the noise has the same variance $\sigma^2=9$.  We note that the coverages of CI build under modLISA and SingleMachine  are much higher.

% Piecewise Data  --------- K=30  &  obs = 20,000
\begin{table*}[h]
\footnotesize
\centering
\caption{Comparing test data  RMSE and coverage of 95\% credible intervals for piecewise $f(x)$ with $N=20,000$ and $K=30$.}
\label{piecewise}
{
\begin{tabular}{crrrrc}
\hline \\
Method  &  \multicolumn{1}{c}{TestRMSE} & \multicolumn{1}{c}{TestCredCov}   \\ \\
\hline
$CMC$ & 1.35  & 100 \%  \\ 
$LISA ~(unif ~ wgh)$ & 0.94 & 100 \% \\ 
$modLISA~(wgh~avg)$ & 0.24  & 90.16 \%\\ 
$SingleMachine $ & 0.15 & 98.76 \%\\ 
 \hline
\end{tabular}
}
\end{table*}

\begin{figure}[!h]
\centering
\begin{subfigure}{.4\textwidth}
  \centering
  \includegraphics[width=1\linewidth]{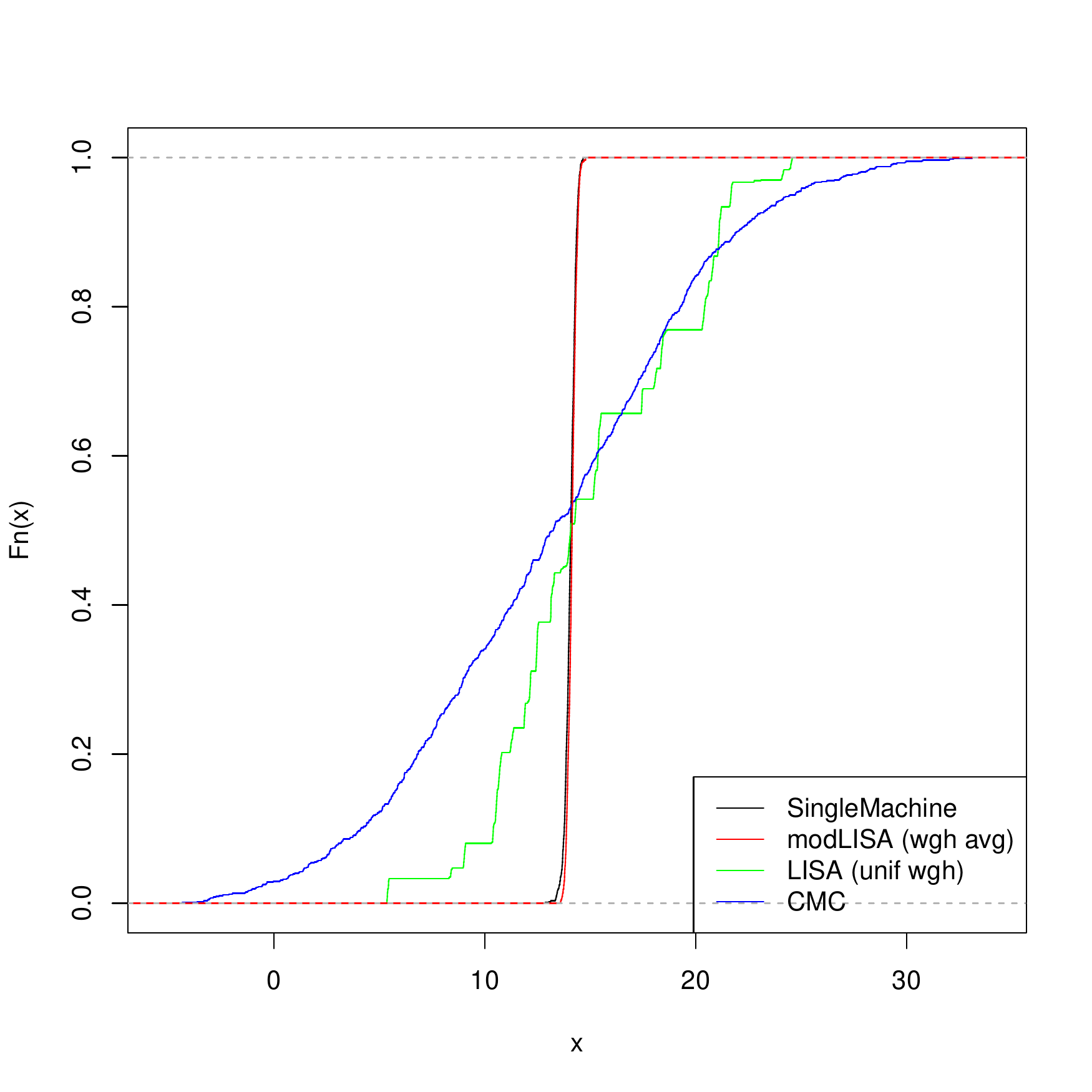}
  \caption{Test $x^*=1000$, $f(x^*) = 13.8$}
  \label{fig:all760}
\end{subfigure}
\begin{subfigure}{.4\textwidth}
  \centering
  \includegraphics[width=1\linewidth]{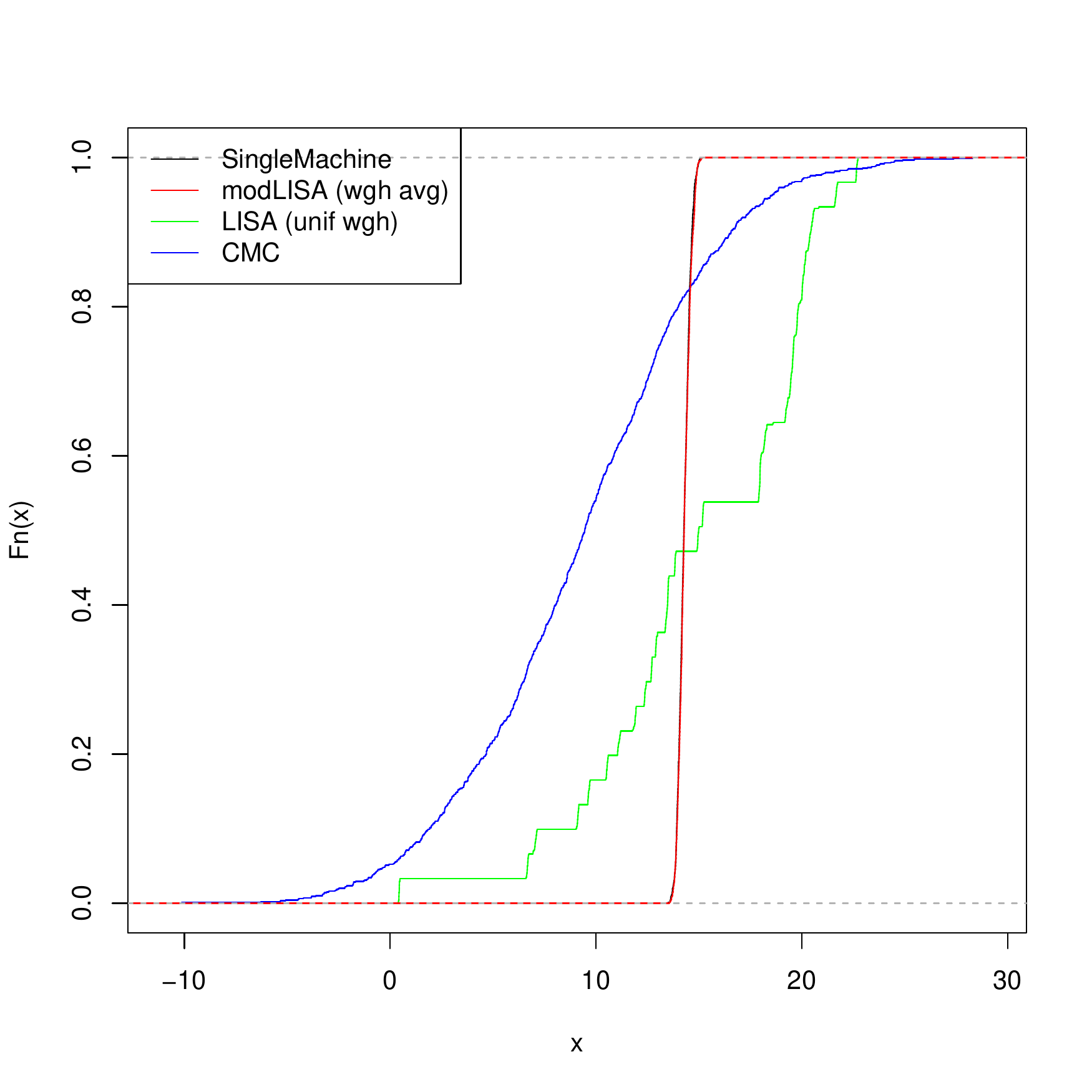}
  \caption{Test $x^* =  2000$, $f(x^*) = 14.4$}
  \label{fig:all999}
\end{subfigure}
\begin{subfigure}{.4\textwidth}
  \centering
  \includegraphics[width=1\linewidth]{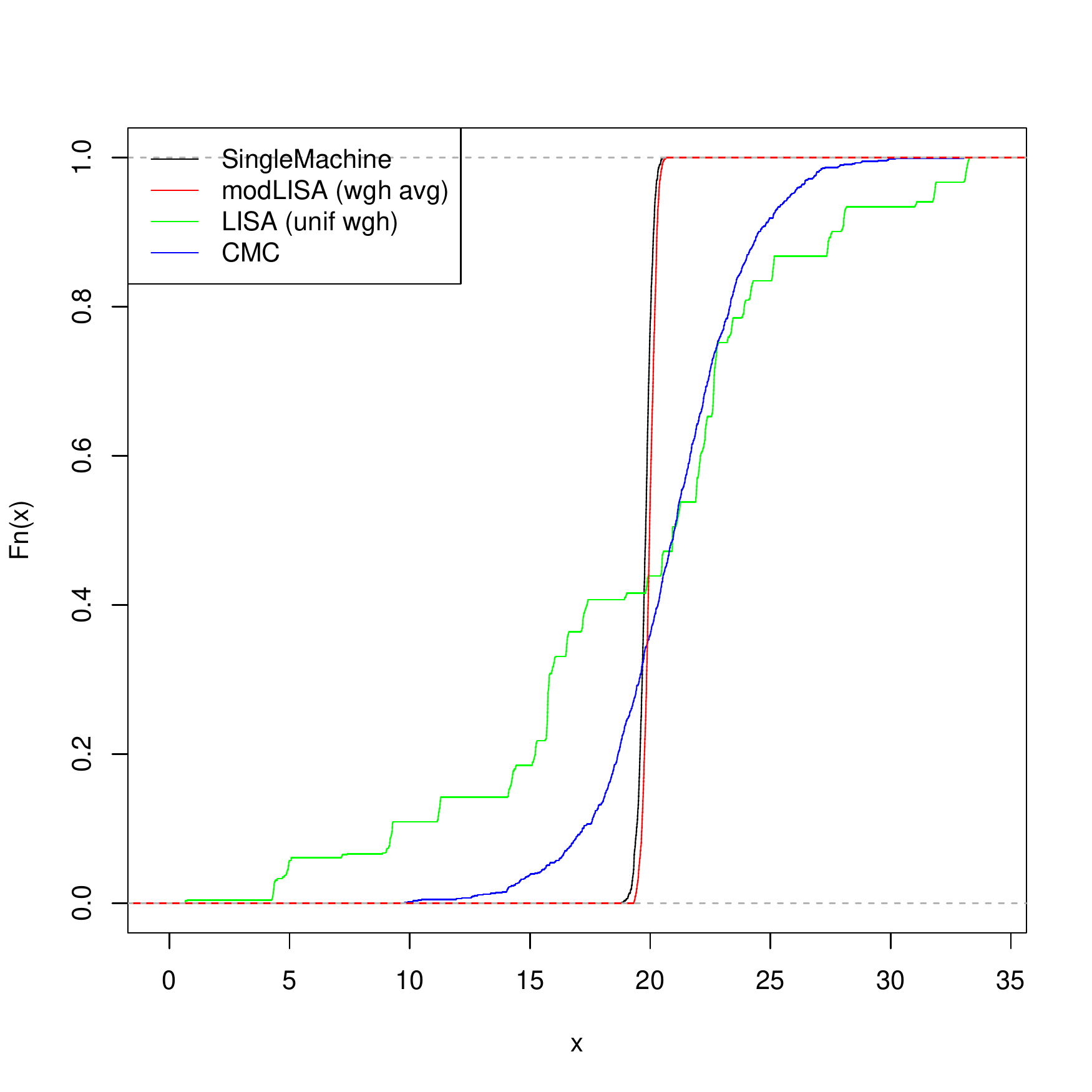}
  \caption{Training $x = 999$, $f(x) = 19.8$}
  \label{fig:all2001}
\end{subfigure}
\begin{subfigure}{.4\textwidth}
  \centering
  \includegraphics[width=1\linewidth]{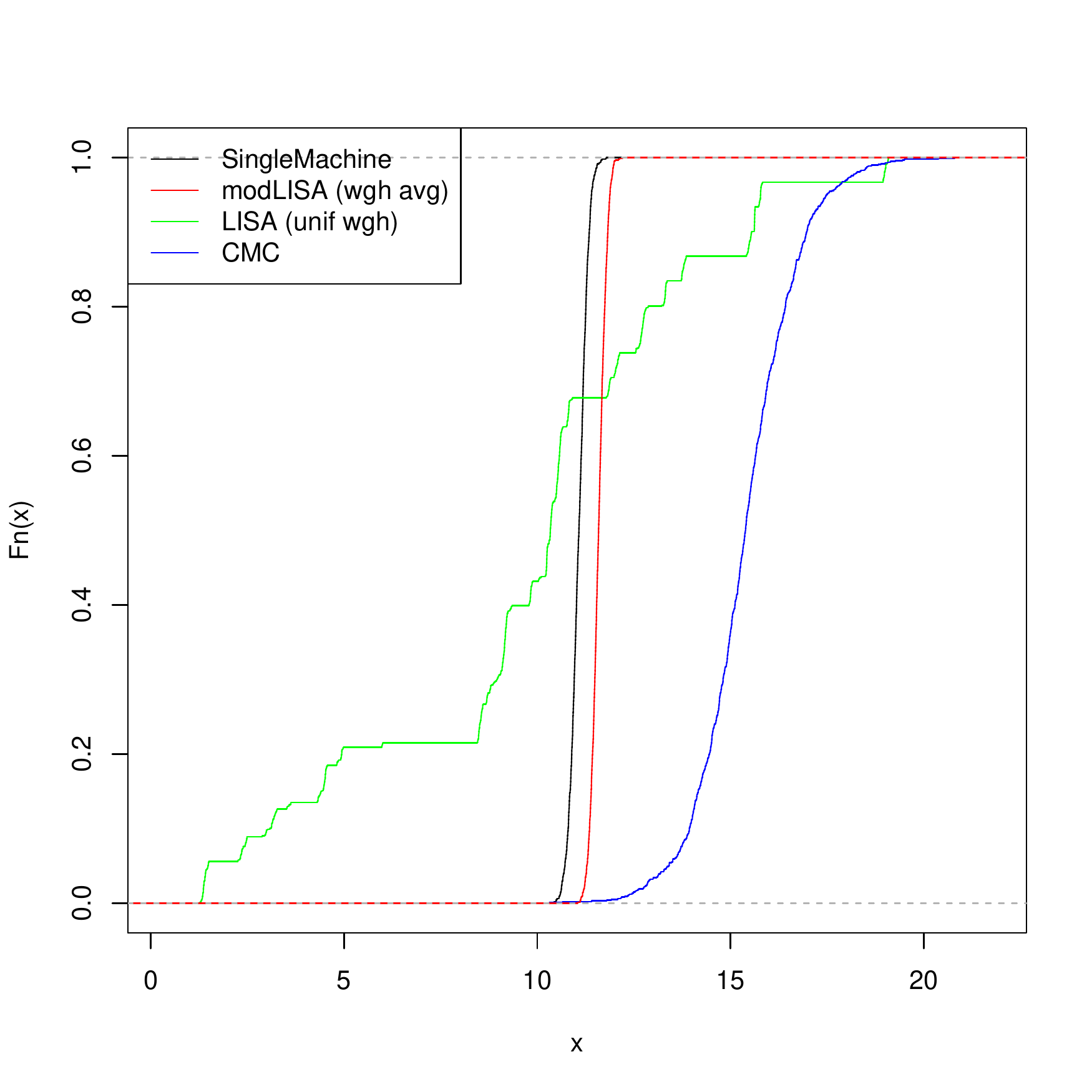}
  \caption{Training $x = 2001$, $f(x) = 11.2$}
  \label{fig:all18000}
\end{subfigure}
\caption{Empirical distribution functions of $\hat{f}(x)$ obtained from  MCMC samples produced by modLISA (red line), LISA (green line), CMC (blue line), and SingleMachine BART (black line) for two different pairs of training and test data. In this example $K=30$.}
\label{fig:4obsall}
\end{figure}

\subsubsection{Comparison with SingleMachine BART}
 In order  to investigate the closeness of posterior samples in each method to the SingleMachine BART, we have plotted in Figure \ref{fig:4obsall}  the empirical distribution functions of $\hat{f}(x)$ generated from each algorithm for two pairs of observations in the training and test dataset. One can see that the empirical distribution functions in LISA and CMC don't match the ones from SingleMachine, and look over-dispersed. However, the empirical distribution functions in modLISA weighted average look much closer to SingleMachine with a  slight shift in location.

\begin{figure}[!h]
\centering
\begin{subfigure}{.45\textwidth}
  \centering
  \includegraphics[width=1\linewidth]{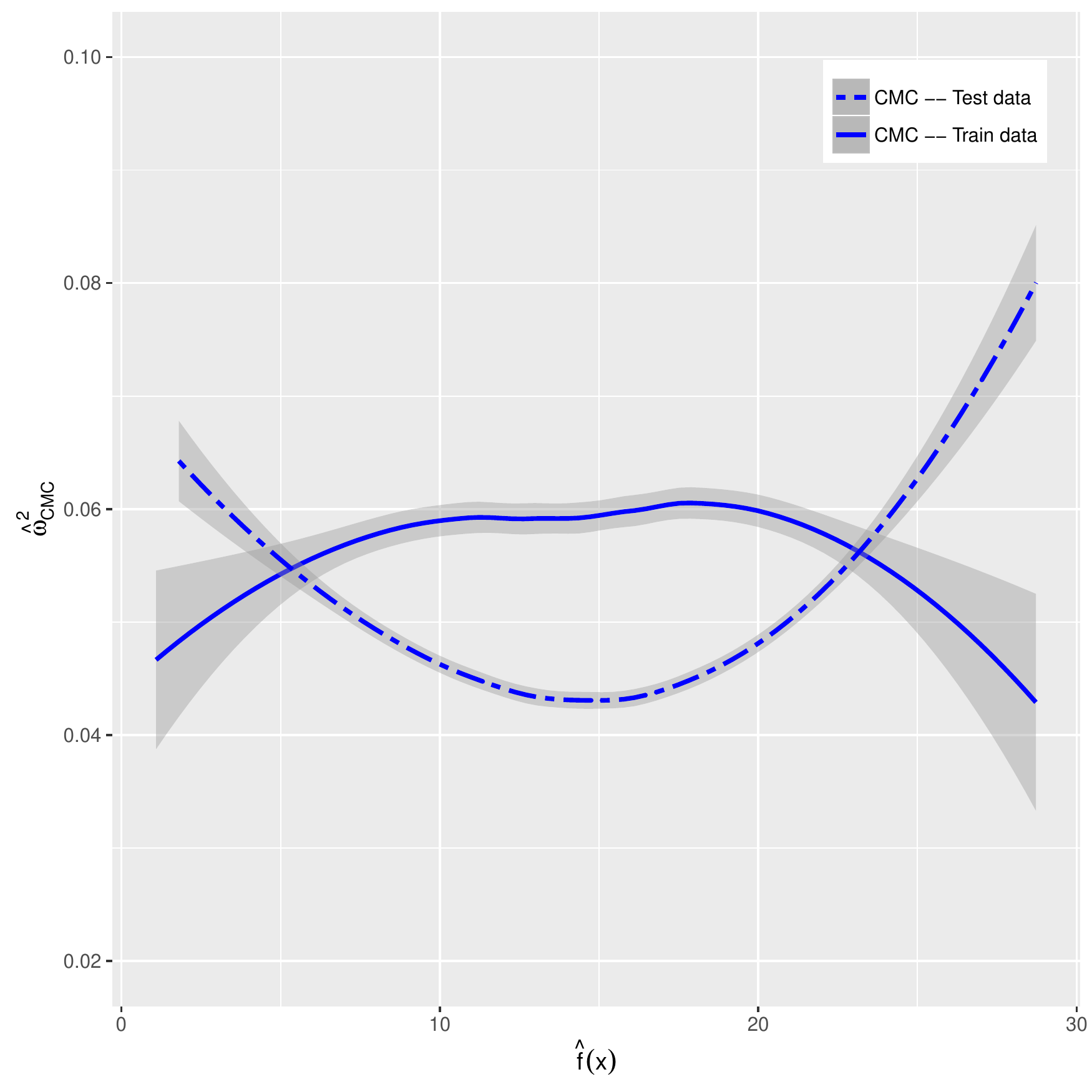}
  \caption{}
  \label{fig:cons}
\end{subfigure}
\begin{subfigure}{.45\textwidth}
  \centering
  \includegraphics[width=1\linewidth]{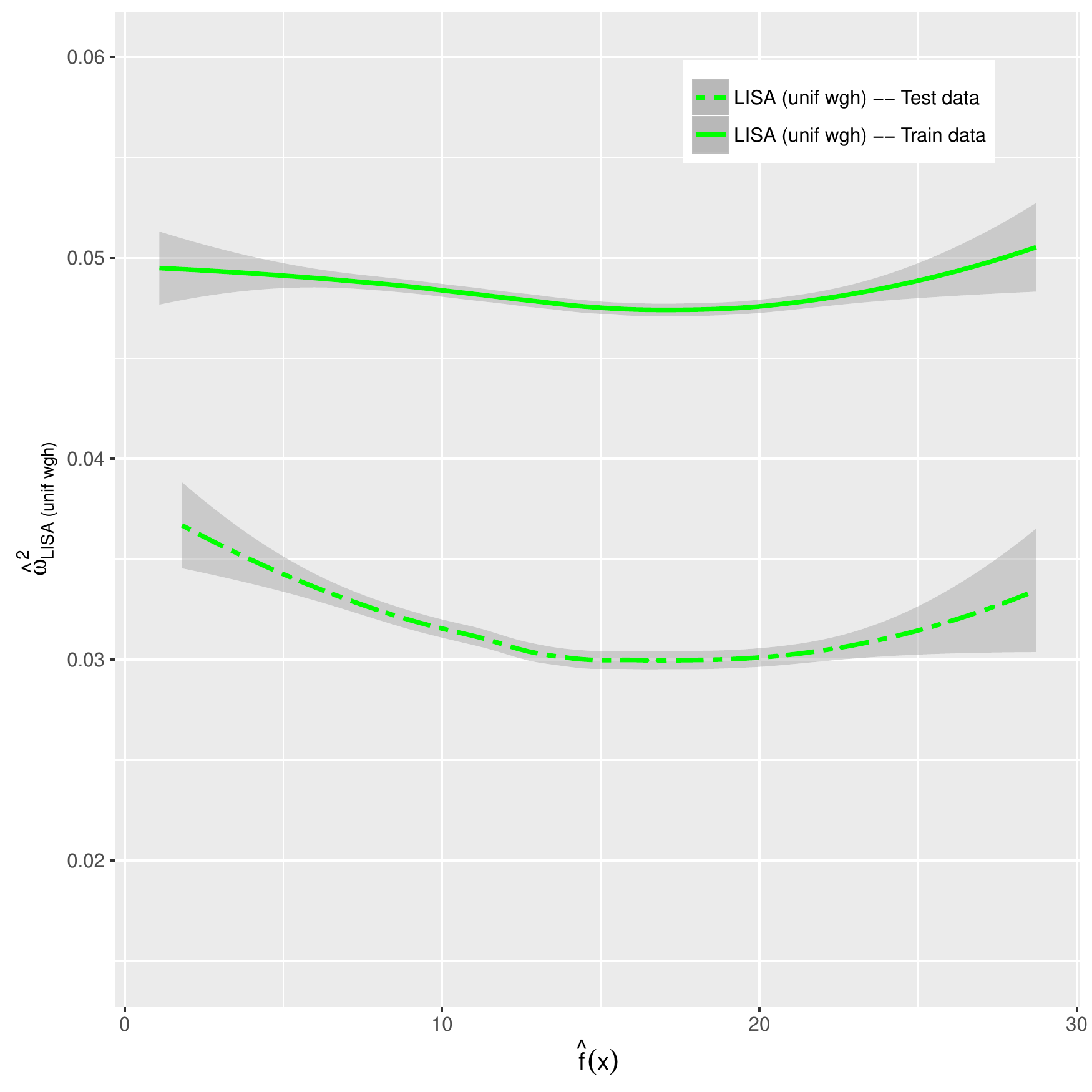}
  \caption{}
  \label{fig:LISAwavg}
\end{subfigure}
\begin{subfigure}{.45\textwidth}
  \centering
  \includegraphics[width=1\linewidth]{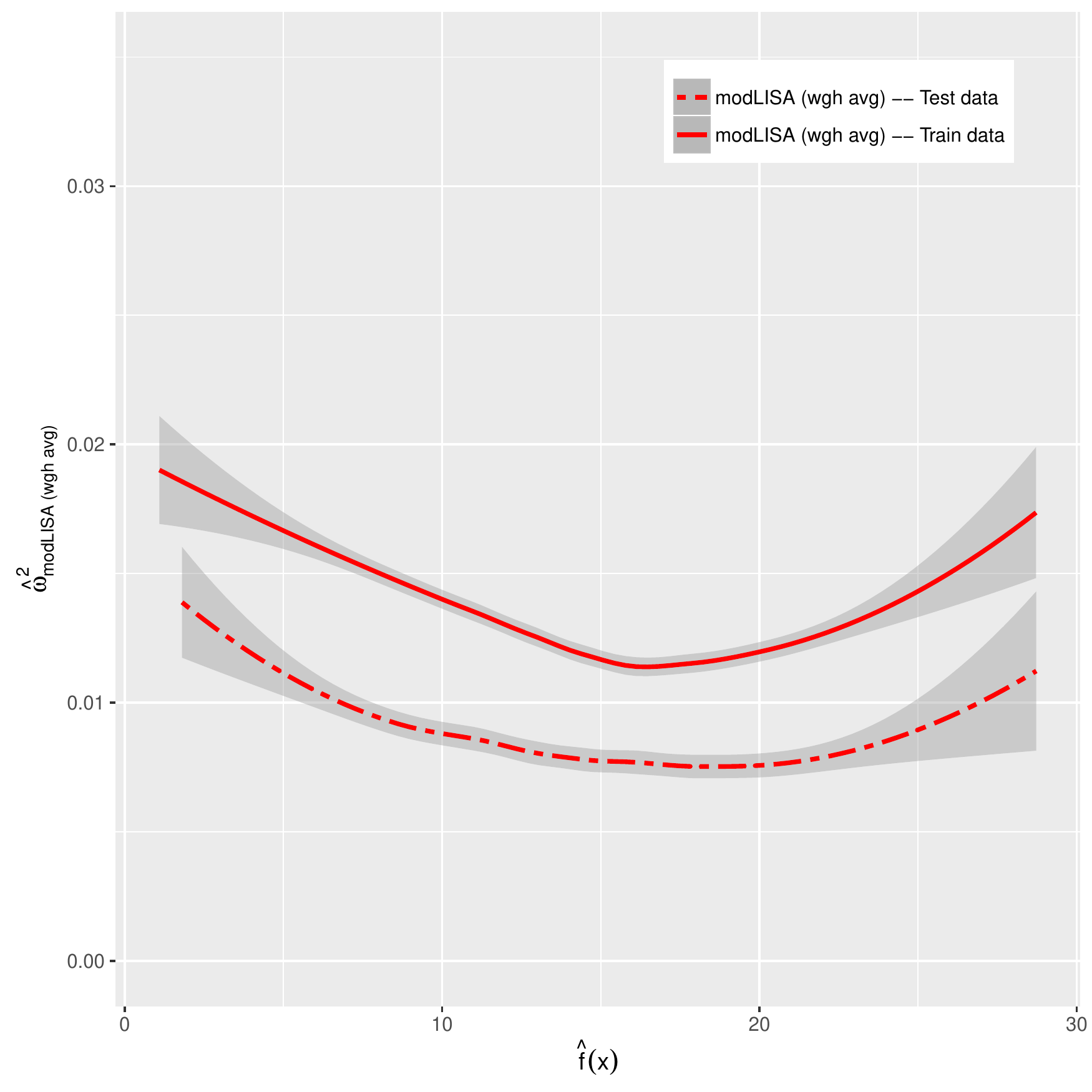}
  \caption{}
  \label{fig:LISAcomb}
\end{subfigure}
\caption{ Blue lines: Fitted polynomial trends (for both train and test data) of average squared difference between empirical distribution functions of SingleMachine and the following: (a) CMC for training (solid line) and test (dot dashed line) data, (b) LISA with uniform weights for training (solid line) and test (dot dashed line) data and (c) modLISA with weighted average for training (solid line) and test (dot dashed line) data. The difference is plotted against  the mean prediction $\hat{f}(x)$ produced by SingleMachine.  Grey areas represent the $95\%$ credible intervals constructed from 100 independent replicates. }
\label{fig:MmodLISA}
\end{figure}

In order to  assess the  performance of the sampling procedures considered, we use the Cram\'er-von Mises distance to assess the difference between empirical distribution functions. This distance is defined to be $\omega^2 = \int_{-\infty}^{\infty}(F_{n}(x) - F(x))^2 dF(x)$ where in our case we assume $F(x) = F_{BART}(x)$ to be the empirical distribution function generated from posterior samples in SingleMachine BART and $F_{n}(x)$ is similarly computed for the alternative method that is considered for comparison. 

Using a set of $T=1000$ equispaced points, we compute the average squared difference between the single machine and all other alternative methods for each observation in the dataset. To illustrate, for LISA we estimate $\omega$ using  ${\hat{\omega}^2}_{LISA} = \frac{1}{T}\sum_{j=1}^{T}(F_{LISA}(t_j) - F_{BART}(t_j))^2$.

Figure \ref{fig:MmodLISA} is comparing the fitted polynomial trends of $\hat{\omega}^2$ (in each method) versus mean predicted $\hat{f}(x)$ in SingleMachine with their corresponding $95\%$ credible regions (for both train and test data). Clearly in LISA and modLISA, there are small variations around the trends with no significant changes in values of $\hat{\omega}^2$ among different mean predicted $\hat{f}(x)$, which specifies consistency within different train or test observations. In addition, the gap between trends from train and test data indicate that the average distance between LISA/modLISA and SingleMachine's distributions are smaller for test data compared to train data. Furthermore, there are still small variations seen around CMC's trends, but with slight changes in values of $\hat{\omega}^2$ among different mean predicted $\hat{f}(x)$, especially for the test dataset which indicates inconsistency within different observations.

\begin{figure}[!h]
\begin{subfigure}{0.5\textwidth}
  \centering
  \includegraphics[width=1\linewidth]{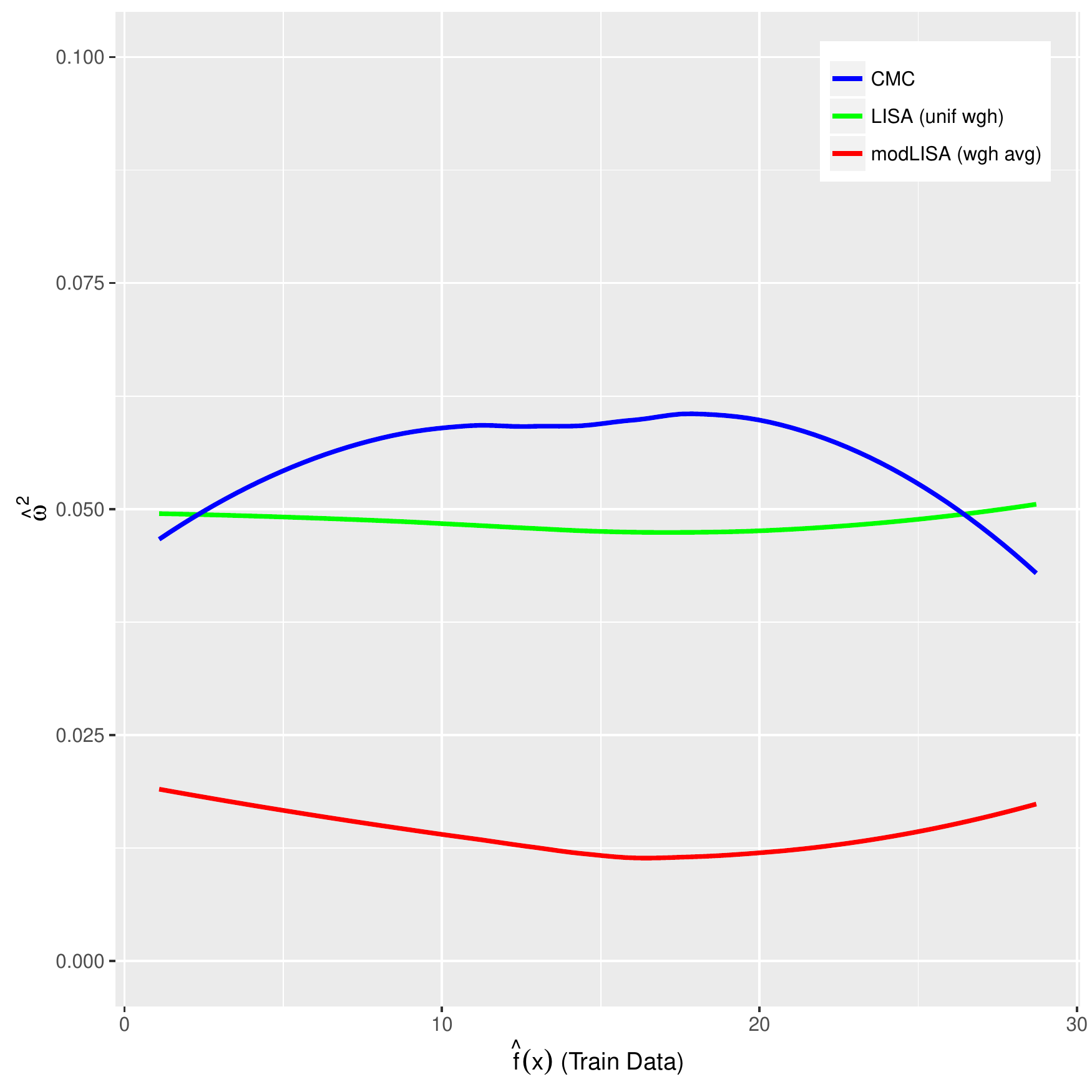}
  \caption{Train data}
\end{subfigure}
\begin{subfigure}{0.5\textwidth}
  \centering
  \includegraphics[width=1\linewidth]{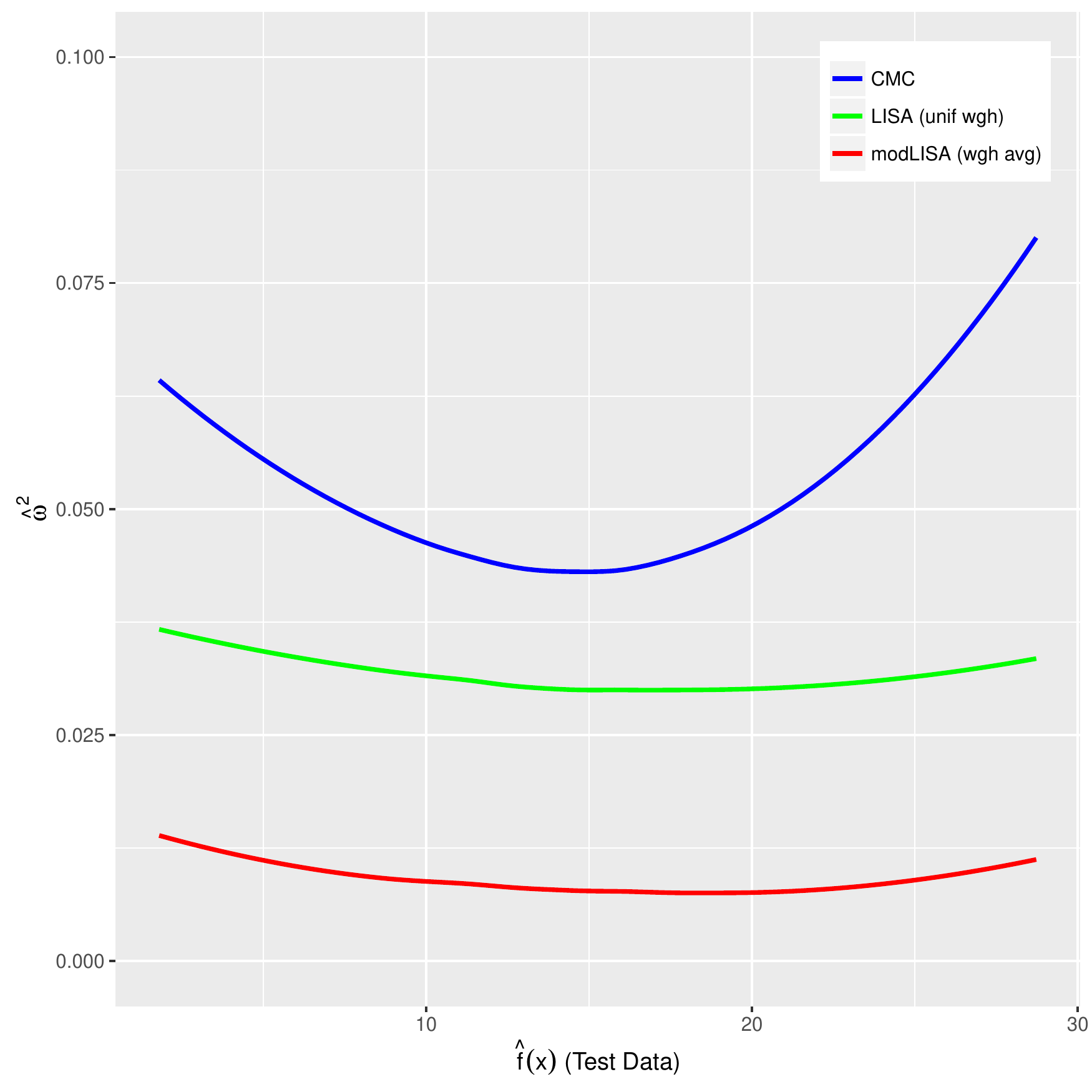}
  \caption{Test data}
\end{subfigure}
\caption{Comparing fitted polynomial trends of average squared difference in empirical distribution functions of each method and SingleMachine, as functions of mean predicted $\hat{f}(x)$ in SingleMachine (for both train and test data).}
\label{fig:allmodLISA}
\end{figure}

To emphasize the difference in performance between modLISA and its competitors, Figure \ref{fig:allmodLISA} shows all the fitted polynomial trends without their credible regions for the train and test data. One can see that there is a large gap between ${\hat{\omega}}^2$ values in modLISA weighted average and other alternative methods (for both train and test data), with modLISA having the lowest value. Thus the weighted average of samples produced by modLISA yields the closest results to SingleMachine. This can also be justified by comparing average ${\hat{\omega}}^2$ over all train observations for each trend which is calculated to be $0.013$ for modLISA that is significantly smaller than $0.059$, $0.048$ for CMC, and LISA, respectively. Similarly, the average ${\hat{\omega}}^2$ over test data are $0.008$, $0.047$, and $0.031$ for modLISA, CMC, and LISA respectively, which again the smallest value is seen in modLISA.
We conclude that modLISA weighted average sample  yields the closest representation of the BART posterior and exhibits the best performance compared to alternative methods.

At last we  compare run time per iteration for each method so we can draw some conclusions regarding the overall efficiency.
 
\subsubsection{\textbf{Run Time Comparisons}}
The main goal of methods such as LISA and CMC was to reduce run times regarding big data applications. Here we have compared average run times per iteration (from one processor) for each method using our implementation of BART. 
\begin{table*}[h]
\centering
\caption{Running times for CMC, LISA, modLISA and SingleMachine when $K=30$.}
\label{speed}
\begin{tabular}{crrc}
\hline \\
Method & \multicolumn{1}{c}{Avg Time per iteration (Secs)}  & \multicolumn{1}{c}{Speed-up}  \\ \\
\hline
$CMC$ & 11.99 & 31\% \\
$LISA$ & 5.04 & 71\% \\
$modLISA$ & 1.81 & 90\% \\
$SingleMachine$ & 17.28 & ----- \\
 \hline
\end{tabular}
\end{table*}

As it is seen in Table \ref{speed}, modLISA, LISA and CMC with $K=30$ are all faster compared to SingleMachine since they are influenced by the smaller subsets of data used. However, since LISA and CMC generate much larger trees, they become slower compared to modLISA which is the fastest method. We have also reported the speed-up percentages with respect to SingleMachine, which is defined to be $(1-t/17.28) \times 100\%$ where $t$ is the average time per iteration in each method. Clearly, CMC shows the smallest speed-up ($31\%$) while modLISA has the highest ($90\%$).

\subsection{Additional Considerations}

\subsubsection{\textbf{Effect of \textbf{$N$} (number of training data) on Posterior Accuracy}}
To see how the number of training data ($N$) can effect the posterior accuracy, we have examined the performance of all methods when $N$ is increased to 60,000 while we keep the same number of batches $K=30$. Tables \ref{60Kdata} shows the results of 1000 posterior samples generated from fitting the BART model to the training set with additional 5000 data considered as test cases.
% K=30  &  obs = 60,000
\begin{table*}[h]
\footnotesize
\centering
\caption{Tree sizes, estimates and 95\% credible intervals for $\sigma^2$, RMSE for training data (TrainRMSE) of size $N=60,000$ and for test data (TestRMSE)  of size $5,000$ for each method run with   $K=30$.}
\label{60Kdata}
\begin{tabular}{crrrrrc}
\hline \\
Method  & \multicolumn{1}{c}{TrainRMSE} & \multicolumn{1}{c}{TestRMSE} & \multicolumn{1}{c}{Tree Nodes} & \multicolumn{1}{c}{Avg $\hat{\sigma}^2$} & \multicolumn{1}{c}{95\% CI for $\sigma^2$}   \\ \\
\hline
$CMC$ & 2.85 & 5.56 & 983 & 0.48 & {[0.30 ~,~ 0.66]} \\ 
$LISA ~(unif ~ wgh)$ & 1.17 & 1.19 & 125 & 0.0003 & {[0.00031 ~,~ 0.00035]} \\
$modLISA~(wgh~avg)$ & 0.41 & 0.42  & 7 & 8.82 & {[8.79 ~,~ 8.86]} \\
$SingleMachine $ & 0.41 & 0.41  & 11 & 9.04 & {[8.94~,~ 9.16]} \\ 
 \hline
\end{tabular}
\end{table*}

% K=30  &  obs = 60,000 ---------COVERAGE
\begin{table*}[h]
\footnotesize
\centering
\caption{Average coverage for  95\% credible intervals constructed for training (TrainCredCov) and test (TestCredCov) data and 95\% prediction intervals constructed for training (TrainPredCov) and test (TestPredCov)  data. The prediction interval coverage is estimated based on 1000 iid samples,  $N=60,000$ and $K=30$.}
\label{Cov60Kdata}
{
\begin{tabular}{crrrc}
\hline \\
Method  &  \multicolumn{1}{c}{TrainPredCov} & \multicolumn{1}{c}{TestPredCov} & \multicolumn{1}{c}{TrainCredCov} & \multicolumn{1}{c}{TestCredCov}  \\ \\
\hline
$CMC$ & 25.74 \% & 17.28 \%  & 51.37  \% & 85.92 \% \\ 
$LISA ~(unif ~ wgh)$ & 0.84 \% & 0.84 \% & 100 \% & 100 \%\\ 
$modLISA~(wgh~avg)$ &  94.54 \% & 94.53 \% & 53.68 \% & 52.68 \%\\ 
$SingleMachine $ &  94.83 \% & 94.84 \% & 57.79 \% & 58.90  \%\\ 
 \hline
\end{tabular}
}
\end{table*}

Unsurprisingly,  Tables  \ref{modLISA} and \ref{60Kdata} show that the  RMSE for training and test data in LISA, modLISA, and SingleMachine decrease as $N$ increases. More importantly, while  LISA and CMC estimates for $\sigma^2$ get worse, modLISA generates more accurate estimates of $\sigma^2$ with a larger $N$. 

Trees have a stable size in modLISA, but tend to grow larger in CMC and LISA, as $N$ increases. 
{Table \ref{Cov60Kdata} shows that coverage of PI decreases in CMC and LISA, but increases in modLISA and SingleMachine for larger training data.  We find it particularly promising that  modLISA competes with SingleMachine for larger $N$. Note that, coverage of CI in LISA and CMC are still unreliable because of their over-dispersion, while in modLISA and SingleMachine they decrease as N increases, which is reasonable since larger sample size creates narrow CI that are around a biased $f(x)$ estimate, as discussed in the previous section. Overall, as $N$ increases, modLISA seems to be a more reliable method as it shows a better performance compared to all other alternatives.}

\subsubsection{\textbf{Effect of \textbf{$K$} (number of batches) on Posterior Accuracy}}
To examine the effect of $K$ on posterior accuracy, we have generated 1000 posterior draws for training data of size $N=20,000$   and $K=10$. The test data sample is of size 5,000.

The results are shown in Table \ref{k10}.
%K=10  &  obs = 20,000
\begin{table*}[!h]
\footnotesize
\centering
\caption{Tree sizes, estimates and 95\% credible intervals for $\sigma^2$, RMSE for training data (TrainRMSE) of size $N=20,000$ and for test data (TestRMSE)  of size $5,000$ for each method run with   $K=10$.}
\label{k10}
\begin{tabular}{crrrrrc}
\hline \\
Method  & \multicolumn{1}{c}{TrainRMSE} & \multicolumn{1}{c}{TestRMSE} & \multicolumn{1}{c}{Tree Nodes} & \multicolumn{1}{c}{Avg $\hat{\sigma}^2$} & \multicolumn{1}{c}{95\% CI for $\sigma^2$}   \\ \\
\hline
$CMC$ & 2.92 & 3.18  & 951 & 0.73 & {[0.57 ~,~ 0.90]} \\        
$LISA ~(unif ~ wgh)$ & 1.70 & 1.78  & 131 & 0.001 & {[0.0010~,~ 0.0012]} \\   
$modLISA~(wgh~avg)$ & 0.46  & 0.47 & 7 & 8.69 & {[8.61 ~,~ 8.77]} \\ 
$SingleMachine $ & 0.55 & 0.56 & 7 & 9.04 &  {[8.85 ~,~ 9.21]}  \\ 
 \hline
\end{tabular}
\end{table*}

%K=10  &  obs = 20,000------COVERAGE
\begin{table*}[h]
\footnotesize
\centering
\caption{Average coverage for  95\% credible intervals constructed for training (TrainCredCov) and test (TestCredCov) data and 95\% prediction intervals constructed for training (TrainPredCov) and test (TestPredCov)  data. The prediction interval coverage is estimated based on 1000 iid samples,  $N=20,000$ and $K=10$.}
\label{Covk10}
{
\begin{tabular}{crrrrrc}
\hline \\
Method  &  \multicolumn{1}{c}{TrainPredCov} & \multicolumn{1}{c}{TestPredCov} & \multicolumn{1}{c}{TrainCredCov} & \multicolumn{1}{c}{TestCredCov}  \\ \\
\hline
$CMC$ & 31.08 \% & 29.83 \%  & 48.18 \% & 99.80  \% \\ 
$LISA ~(unif ~ wgh)$ &  1.44 \% & 1.43 \%  & 99.98 \% & 99.96 \%\\ 
$modLISA~(wgh~avg)$ &   94.30 \% & 94.29 \%  & 71.08 \% & 70.32 \%\\ 
$SingleMachine $ &  94.67 \% & 94.65 \% & 71.58 \% & 71.54 \%\\ 
 \hline
\end{tabular}
}
\end{table*}

As $K$ decreases, the performance of LISA and CMC drops while modLISA generates stronger results, which is intuitively expected as each batch is larger and closer to the full sample when $K$ is smaller. We also note  the improvement of modLISA  over  SingleMachine in terms of RMSE. 
{In addition, Table \ref{Covk10} shows that the PI and CI coverages for modLISA and SingleMachine are very close.}

\subsection{\textbf{Varying the Underlying Model -- Different $f(x)$}}
Consistency in performance of modLISA can also be seen when the underlying model is changed. For instance, we also considered a sample of size  20,000 using  
\beq
f(x) = 3 \sqrt{x_1} -2{{x_2}^{2}} + 5 {x_3}{x_4},
\label{m2}
\eeq
 where $x=(x_{1},\ldots,x_{4})$ is a four-dimensional input vector that is simulated independently from a $U(0,1)$ and $y \sim \mathcal{N}(f(x),\sigma^2)$ with $\sigma^2=1$. Additional 5000 data have also been simulated as test cases. Similarly, by fitting this newly simulated dataset to each method with $K=30$, we have generated 1000 posterior samples with results averaged across three different realizations of data shown in {Tables \ref{f2} and \ref{Covf2}}.

%%%%% new f2 averaged over 3 realizations
\begin{table*}[!h]
\footnotesize
\centering
\caption{Tree sizes, estimates and 95\% credible intervals for $\sigma^2$, RMSE for training data (TrainRMSE) of size $N=20,000$ generated from \eqref{m2} and for test data (TestRMSE)  of size $5,000$ for each method run with $K=30$. Results are averaged over three different data replications.}
\label{f2}
\begin{tabular}{crrrrrc}
\hline \\
Method  & \multicolumn{1}{c}{TrainRMSE} & \multicolumn{1}{c}{TestRMSE}  & \multicolumn{1}{c}{Tree Nodes} & \multicolumn{1}{c}{Avg $\hat{\sigma}^2$} & \multicolumn{1}{c}{95\% CI for $\sigma^2$}  \\ \\
\hline
$CMC$ & 0.89 & 0.76  & 614 & 0.21 & {[0.18~,~ 0.34]}  \\ 
$LISA ~(unif ~ wgh)$ & 0.32 & 0.33  & 57 & 0.0001 & {[0.000083 ~,~0.000103]} \\ 
$modLISA~(wgh~avg)$ & 0.11  & 0.11  & 7 & 0.88 & {[0.87~,~ 0.89]} \\ 
$SingleMachine $ & 0.14 & 0.14 & 7 & 1.00 & {[0.99~,~1.03]}  \\
 \hline
\end{tabular}
\end{table*}

% new f2  --------- K=30  &  obs = 20,000 -------COVERAGE
\begin{table*}[h]
\footnotesize
\centering
\caption{Average coverage for  95\% credible intervals constructed for training (TrainCredCov) and test (TestCredCov) data  and 95\% prediction intervals constructed for training (TrainPredCov) and test (TestPredCov)  data generated from \eqref{m2}. The prediction interval coverage is estimated based on 1000 iid samples,  $N=20,000$ and $K=30$. Results are averaged over three different data replications.}
\label{Covf2}
{
\begin{tabular}{crrrc}
\hline \\
Method  &  \multicolumn{1}{c}{TrainPredCov} & \multicolumn{1}{c}{TestPredCov} & \multicolumn{1}{c}{TrainCredCov} & \multicolumn{1}{c}{TestCredCov}  \\ \\
\hline
$CMC$ & 49.74 \% & 52.97 \% & 84.16  \% & 100 \% \\ 
$LISA ~(unif ~ wgh)$ & 1.50 \% & 1.49 \% & 100 \% & 100 \%\\ 
$modLISA~(wgh~avg)$ &  93.07 \% & 93.18 \%   & 82.88 \% & 83.50 \%\\ 
$SingleMachine $ & 94.82 \% & 94.81 \%  & 79.13 \% & 78.47 \%\\ 
 \hline
\end{tabular}
}
\end{table*}

Again modLISA {outperforms all alternative methods}, and its performance is  closest to SingleMachine. This confirms the previous simulation results and allows us to conclude that modLISA is a more reliable method for BART models with large datasets.

In the next section we will apply modLISA weighted average BART to a large socio-economic study.

\subsection{\textbf{Real Data Analysis}}
%\subsection{\textbf{2013 American Community Survey (ACS) data:}}
The American Community Survey (ACS) is a growing survey from the US Census Bureau and the Public Use Microdata Sample (PUMS) is a sample of responses to ACS which consists of various variables related to people and housing units \citep[see][2013]{dataset}. Considering the person-level data from PUMS 2013, we would like to predict a person's total income based on variables such as sex, age, education, class of worker, living state, and citizenship status. We have collected information related to people who are employed and have total income of at least \$5000 with education level of either Bachelor's degree, Master's degree, or a PhD which resulted in $437,297$ observations. We randomly divided the dataset into approximately $80\%$ training and $20\%$ testing sets, with $K=100$ batches considered for splitting the training data to apply modLISA. Computations were performed on the GPC supercomputer at the SciNet HPC Consortium \citep{scinet} using 100 cores, each running on $3,500$ observations. Considering the logarithm of total income for each person as the response variable, we have ran modLISA with weighted average and SingleMachine BART on this dataset for 1500 iterations (since SingleMachine is very slow) and discarded the first 1000 draws which resulted in 500 posterior samples. Table \ref{ACS} contains the results of Test RMSE as well as average post burn-in $\sigma^2$ estimates and tree sizes.
\begin{table*}[h]
\centering
\caption{Perfomance summaries computed from 1000 posterior samples generated from modLISA with $K=100$ and SingleMachine BART on PUMS 2013 test data.}
\label{ACS}
\begin{tabular}{crrrrrc}
\hline \\
Method  & \multicolumn{1}{c}{TestRMSE} & \multicolumn{1}{c}{Avg $\hat{\sigma}^2$} & \multicolumn{1}{c}{Tree Nodes} & \multicolumn{1}{c}{Speed-up}   \\ \\
\hline
$modLISA~(wgh~avg)$ & 0.71 &  0.488 & 7  & 90\% \\ 
$SingleMachine$ & 0.70 & 0.485  & 23  & -- \\ 
 \hline
\end{tabular}
\end{table*}

%However, modLISA shows significant advantage as it has $90\%$ speed-up with respect to SingleMachine.

\begin{table*}[h]
\centering
\caption{Average acceptance rates of tree proposal moves.}
\label{accmodLISA_ACS}
\begin{tabular}{crrrc}
\hline \\
Method & \multicolumn{1}{c}{GROW} & \multicolumn{1}{c}{PRUNE}   & \multicolumn{1}{c}{CHANGE} \\ \\
\hline
$modLISA$ & 10 \% & 11 \% & 14 \% \\ %9\% &   12\% &  13\%  \\ 
$SingleMachine$ & 8\% &   7\% &  7\%  \\ 
\hline
\end{tabular}
\end{table*}

One can see that Test RMSE in modLISA is similar to the one from SingleMachine, but with a $90\%$ speed-up of modLISA over SingleMachine.  The speed-up can be  explained by the larger acceptance probabilities  and by the smaller tree sizes reported in  Tables \ref{accmodLISA_ACS} and \ref{ACS}, respectively. The 90\% speedup is important for applications like the one considered here, as it takes more than a day to simulate $1,500$ samples from the posterior using SingleMachine. The result indicate the potential of the proposed   method for reducing computational costs while producing accurate predictions.

\section{\textbf{Discussion}}

The challenge of using MCMC algorithms to  sample posterior distributions obtained from a massive sample of observations  is a serious one. 

In this paper, we introduced a new method based on the idea of randomly dividing the  data into batches and  drawing samples from each of the resulting sub-posteriors  independently and  {in parallel} on different machines.
We propose a novel way to define the sub-posteriors and we develop a strategy to combine the samples produced by each batch analysis for the   important  class of  Bayesian Additive Regression Trees Models. For this model, the proposed methodology performs very well and shows reduction in computation time that are as high as 90\%.  

In future work we would like to find a procedure for combining the sub-posterior samples that will make LISA easy to adapt to a wide variety of models.   We also hope that our paper will stimulate the research into this type of  divide-and-conquer approaches for Big Data MCMC and will expand the research on how to construct the batch-specific sub-posteriors along with novel strategies of combining or weighting the samples obtained from each batch analysis.

\section{Acknowledgement}
We thank the Editor, Grace Yi, the Guest Editor, Richard Lockhart, the Associate Editor and three anonymous referees for helpful suggestions that have greatly improved the paper. This work has been supported by NSERC of Canada grants to RVC and JSR.

%\bibliography{references}

\begin{thebibliography}{15}
\providecommand{\natexlab}[1]{#1}
\providecommand{\url}[1]{\texttt{#1}}
\expandafter\ifx\csname urlstyle\endcsname\relax
  \providecommand{\doi}[1]{doi: #1}\else
  \providecommand{\doi}{doi: \begingroup \urlstyle{rm}\Url}\fi

\bibitem[Brooks et~al.(2011)Brooks, Gelman, Jones, and Meng]{MCMChandbook}
S.~Brooks, A.~Gelman, G.~L. Jones, and X.~Meng, editors.
\newblock \emph{Handbook of Markov Chain {M}onte {C}arlo}.
\newblock Taylor \& Francis, 2011.

\bibitem[Chipman et~al.(1998)Chipman, George, and McCulloch]{cart}
Hugh~A Chipman, Edward~I George, and Robert~E McCulloch.
\newblock {B}ayesian CART model search.
\newblock \emph{Journal of the American Statistical Association}, 93\penalty0
  (443):\penalty0 935--948, 1998.

\bibitem[Chipman et~al.(2010)Chipman, George, and McCulloch]{bart}
Hugh~A Chipman, Edward~I George, and Robert~E McCulloch.
\newblock Bart: {B}ayesian additive regression trees.
\newblock \emph{The Annals of Applied Statistics}, pages 266--298, 2010.

\bibitem[Craiu and Rosenthal(2014)]{radu_rosenthal}
Radu~V Craiu and Jeffrey~S Rosenthal.
\newblock {B}ayesian computation via Markov chain {M}onte {C}arlo.
\newblock \emph{Annual Review of Statistics and Its Application}, 1:\penalty0
  179--201, 2014.

\bibitem[Friedman(1991)]{friedman}
Jerome~H Friedman.
\newblock Multivariate adaptive regression splines.
\newblock \emph{The annals of statistics}, pages 1--67, 1991.

\bibitem[Kapelner and Bleich(2013)]{bartmachine}
Adam Kapelner and Justin Bleich.
\newblock bartmachine: Machine learning with {B}ayesian additive regression
  trees.
\newblock \emph{arXiv preprint arXiv:1312.2171}, 2013.

\bibitem[Laskey and Myers(2003)]{laskey}
Kathryn~Blackmond Laskey and James~W Myers.
\newblock Population Markov chain {M}onte {C}arlo.
\newblock \emph{Machine Learning}, 50\penalty0 (1-2):\penalty0 175--196, 2003.

\bibitem[Loken et~al.(2010)Loken, Gruner, Groer, Peltier, Bunn, Craig,
  Henriques, Dempsey, Yu, Chen, et~al.]{scinet}
Chris Loken, Daniel Gruner, Leslie Groer, Richard Peltier, Neil Bunn, Michael
  Craig, Teresa Henriques, Jillian Dempsey, Ching-Hsing Yu, Joseph Chen, et~al.
\newblock Scinet: lessons learned from building a power-efficient top-20 system
  and data centre.
\newblock In \emph{Journal of Physics: Conference Series}, volume 256, page
  012026. IOP Publishing, 2010.

\bibitem[Neiswanger et~al.(2013)Neiswanger, Wang, and Xing]{emb}
Willie Neiswanger, Chong Wang, and Eric Xing.
\newblock Asymptotically exact, embarrassingly parallel MCMC.
\newblock \emph{arXiv preprint arXiv:1311.4780}, 2013.

\bibitem[Pratola et~al.(2016)]{pratola}
Matthew~T Pratola et~al.
\newblock Efficient Metropolis--Hastings proposal mechanisms for {B}ayesian
  regression tree models.
\newblock \emph{{B}ayesian Analysis}, 11\penalty0 (3):\penalty0 885--911, 2016.

\bibitem[Rosenthal(2000)]{rosenthal}
Jeffrey~S Rosenthal.
\newblock Parallel computing and {M}onte {C}arlo algorithms.
\newblock \emph{Far east journal of theoretical statistics}, 4\penalty0
  (2):\penalty0 207--236, 2000.

\bibitem[Scott et~al.(2013)Scott, Blocker, Bonassi, Chipman, George, and
  McCulloch]{cons}
Steven~L Scott, Alexander~W Blocker, Fernando~V Bonassi, H~Chipman, E~George,
  and R~McCulloch.
\newblock {B}ayes and big data: The consensus {M}onte {C}arlo algorithm.
\newblock In \emph{International Journal of Management Science and Engineering Management}, 11\penalty0
  (2):\penalty0 78--88, 2016.

\bibitem[{US Bureau of Census}()]{dataset}
{US Bureau of Census}.
\newblock 2013 {ACS 1-YEAR PUMS} data.
\newblock URL \url{http://www.census.gov/programs-surveys/acs/data/pums.html}.

\bibitem[Wang and Dunson(2013)]{weierstrass}
Xiangyu Wang and David~B Dunson.
\newblock Parallelizing {MCMC} via {W}eierstrass sampler.
\newblock \emph{arXiv preprint arXiv:1312.4605}, 2013.

\bibitem[Wilkinson(2006)]{wilkinson}
Darren~J Wilkinson.
\newblock Parallel {B}ayesian computation.
\newblock \emph{Statistics Textbooks and Monographs}, 184:\penalty0 477, 2006.

\end{thebibliography}

\end{document}